\documentclass{article}
\usepackage{iclr2027_conference,times}

\usepackage{amsmath,amsfonts,bm}

\def\eqref#1{equation~\ref{#1}}
\def\1{\bm{1}}

\DeclareMathAlphabet{\mathsfit}{\encodingdefault}{\sfdefault}{m}{sl}
\SetMathAlphabet{\mathsfit}{bold}{\encodingdefault}{\sfdefault}{bx}{n}

\usepackage{hyperref}
\hypersetup{
    colorlinks=true,
    linkcolor=blue!70!black,
    citecolor=blue!70!black,
    urlcolor=blue!70!black
}
\usepackage{url}
\usepackage{amsmath,amssymb,amsfonts}
\usepackage{booktabs}
\usepackage{graphicx}
\usepackage{wrapfig}
\usepackage{subcaption}
\usepackage{multirow}
\usepackage{xcolor}
\usepackage{algorithm}
\usepackage{algorithmic}
\usepackage{enumitem}

\newsavebox{\adjbox}
\newenvironment{adjustbox}[1]{%
  \begin{lrbox}{\adjbox}%
}{%
  \end{lrbox}%
  \resizebox{\linewidth}{!}{\usebox{\adjbox}}%
}

\setlist{nosep,leftmargin=1.5em}

\newtheorem{theorem}{Theorem}

\newtheorem{proposition}[theorem]{Proposition}

\newtheorem{remark}{Remark}

\providecommand{\qed}{\hfill\ensuremath{\blacksquare}}
\newenvironment{proof}[1][Proof]{\noindent\textbf{#1.} }{\qed\par\smallskip}

\title{A Fisher-Geometric View of Multi-Teacher On-Policy Distillation}

\author{Kaiyuan Liu$^{1,4}$, Ziyuan Zhuang$^{4}$\thanks{Project leader.}, Haoxuan Ma$^{2}$, Yannian Gu$^{3}$, Mingfei Song$^{3}$, \\
Rongxiang Weng$^{4}$\thanks{Corresponding author.}, Jieping Ye$^{1}$ \\
\normalfont $^{1}$Zhejiang University \quad $^{2}$Nanjing University \quad $^{3}$Shanghai Jiao Tong University \\
\normalfont $^{4}$Meituan LongCat
}

\iclrfinalcopy

\begin{document}

\maketitle

\begin{abstract}
Consolidating specialized foundation models into a unified generalist is a central goal of post-training. Two approaches are in common use: model merging, which is fast but treated as a heuristic, and multi-teacher on-policy distillation, which matches distributions but demands heavy rollout compute and suffers from cross-task gradient conflict. We show that the two are one objective seen at different orders. Within a shared basin of homologous specialists, the multi-teacher on-policy distillation objective is, to second order, exactly the Fisher-merging objective. Its minimizer is therefore the Fisher merge, a classic model-merging method given by a closed-form combination of the specialists, so \emph{merging is optimal, to second order, as a closed-form stage of distillation}. Building on this identification, we propose \emph{Residual-Governed Branch-and-Merge} (RGBM), a two-phase consolidation procedure. Concretely, a full-parameter Fisher merge is intractable on foundation models, so Phase~1 restricts the merge to the task-vector subspace. Extending tangent-space analyses of merging, we show that the same coefficients that mix task vectors also mix implicit rewards, which lets a residual program on unlabelled prompts solve for them. The program minimizes the worst per-domain residual and reports its energy-normalized values as the \emph{domain residual spectrum}. Phase~2 lets this spectrum govern what may be added: RGBM passivates domains with low normalized residuals, trains isolated branches only for domains above the spectral gap, and writes back the admitted updates with calibrated shrinkage. On two open-source consolidation suites, the Phase~1 merge alone matches or outperforms fully trained multi-teacher distillation, and RGBM improves on it further with zero extra joint-training steps, whereas every un-gated write-back, including continued joint distillation from the merge, falls below it.
\end{abstract}

\section{Introduction}
\label{sec:intro}

Post-training of large language models (LLMs) increasingly produces homologous domain specialists---fine-tuned from a common base via reinforcement learning (RL) on distinct reasoning modalities~\citep{agarwal2024onpolicy,shao2024deepseekmath,deepseek2025r1,moonshot2025kimi,openmopd2026,liu2025efficient}. Serving an ensemble of specialists multiplies memory and latency, making consolidation into a single generalist essential.

Two paradigms are in common use, each with its own shortcomings. \emph{Weight-space merging}~\citep{ilharco2023editing,yadav2023ties,yu2024dare,lu2025twin,tang2024merging} adds to the base model a weighted sum of the specialists' weight updates (their \emph{task vectors}) at zero training cost, but the weights are set by hand and, once the merge is formed, nothing tells which domain it failed to recover or whether the chosen weights still combine the specialists faithfully. \emph{Multi-teacher on-policy distillation}~\citep{agarwal2024onpolicy,openmopd2026,d3mopd2026,gopd2026,li2025policyfiltration} (MOPD) instead trains a student to match each specialist (its \emph{teacher}) on the student's own generations, which repairs what merging misses, but at the cost of rollouts on every domain and with no control over which domains its shared-weight gradients overwrite. Neither paradigm can answer the question the other one raises: what a merge has already recovered, and where distillation is still needed without harming what is already correct. Answering it requires relating the two, which is what this paper does.

\begin{figure}[t]
\centering
\includegraphics[width=0.84\textwidth]{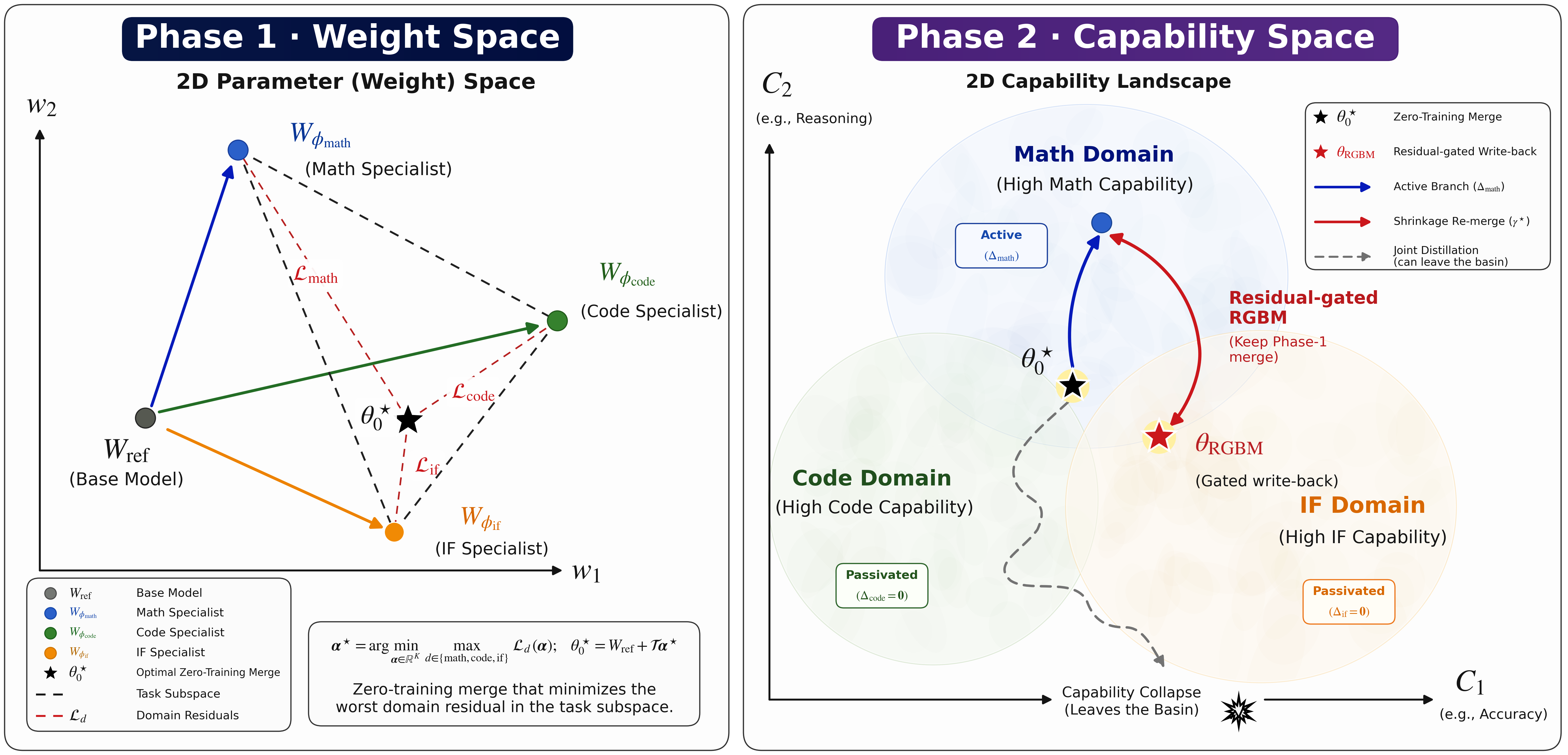}
\caption{\textbf{Merge before you distill (schematic).} (a)~Weight space: the dashed triangle is the task subspace spanned by the Math, Code, and IF task vectors. A Minimax projection on unlabelled prompts yields the zero-training merge $\theta_0^\star=W_{\mathrm{ref}}+\mathcal{T}\boldsymbol{\alpha}^\star$, where $\mathcal{T}$ stacks those task vectors, together with the per-domain residuals $\mathcal{L}_d$ (red segments). Normalizing these residuals by domain reward energy gives the domain residual spectrum. (b)~Capability space: joint distillation from $\theta_0^\star$ can leave the basin (dashed path). The normalized spectrum passivates Code and IF; RGBM writes back only the admitted Math branch, shrunk by $\gamma^\star$ (the largest factor that still passes the linearity test). See Section~\ref{sec:method} for details.}
\vspace*{-2\baselineskip}
\label{fig:teaser}
\end{figure}

Our starting point is that these paradigms optimize the same objective at different orders of approximation (Section~\ref{sec:theory}). Expanded to second order around the specialists in a shared basin, the distillation objective is exactly the Fisher-merging objective of \citet{matena2022merging}, whose minimizer is a weighted average of the specialists with Fisher-information weights; merging is therefore not an alternative to distillation but is optimal, to second order, as its closed-form stage. A full-parameter Fisher merge is intractable on foundation models, so we restrict the merge to the task-vector subspace (Phase~1, described next). Extending tangent-space analyses of merging, we show that, to first order, the same coefficients that mix the task vectors also mix the specialists' \emph{implicit rewards} (the log-probability ratios between each specialist and the base), which lets us obtain them by minimizing the worst reward-space residual on unlabelled prompts. After solving, reward-energy normalization measures the relative share left unrecovered in each domain. The fidelity of this first-order fit provides an intrinsic, label-free test of whether a checkpoint still lies where the expansion holds.

Grounded in this framework, we propose \emph{Residual-Governed Branch-and-Merge} (RGBM) as shown in Figure~\ref{fig:teaser}. In Phase~1, RGBM determines the combination weights by minimizing the worst per-domain residual via a small convex program on unlabelled prompts, yielding the zero-training merged model. It then normalizes the residuals by their domain reward energies to form the \emph{domain residual spectrum}. This spectrum directly dictates what, if anything, should follow the merge. Standard joint distillation tends to degrade the merged model because gradients from domains with low normalized residuals corrupt shared weights. To prevent this interference, Phase~2 leaves those domains untouched and trains isolated branches only for domains above a closed-form spectral gap. Each admitted branch update is then written back, scaled by the largest factor that passes the linearity test.

Empirically, we validate our framework across two multi-teacher consolidation benchmarks: Open-MOPD~\citep{openmopd2026} (SmolLM3-3B with Math, Code, and Instruction Following specialists) and G-OPD~\citep{gopd2026} (Qwen3-4B with Math and Code specialists). On Open-MOPD, the zero-training Phase~1 merge alone achieves a Macro average of $32.52$, outperforming the official $600$-step distilled student ($30.93$) without requiring any rollout compute. By passivating domains with low normalized residuals and training an isolated branch solely for Math, Phase~2 further elevates the Macro score to $33.03$ with zero joint-training steps. Crucially, un-gated baselines such as TIES~\citep{yadav2023ties}, DARE~\citep{yu2024dare}, and continued joint distillation all fall below the merge, confirming that their updates can corrupt shared parameters. The same procedure transfers seamlessly to G-OPD, where RGBM surpasses the strongest trained baseline, demonstrating the generality of residual-governed consolidation.

\paragraph{Contributions.}
(1)~\textbf{Theoretical Bridge:} we establish a connection between multi-teacher on-policy distillation and model merging, showing that the MOPD objective coincides to second order with Fisher merging \citep{matena2022merging}, and derive an implicit-reward homomorphism that yields a label-free linear basin diagnostic.
(2)~\textbf{Two-Phase RGBM Framework:} we propose \emph{Residual-Governed Branch-and-Merge} (RGBM). Phase~1 computes the zero-training merge from a residual program on unlabelled prompts and forms an energy-normalized spectrum for cross-domain comparison. Phase~2 passivates domains below the spectral gap, trains isolated branches exclusively for domains above it, and writes admitted updates back via calibrated shrinkage.
(3)~\textbf{Empirical Validation and Protection Mechanism:} across two model families (SmolLM3-3B and Qwen3-4B), Phase~1 alone matches or beats fully trained distillation, and Phase~2 advances performance further with zero joint-training steps. Crucially, un-gated write-backs fall below the merge, confirming that RGBM succeeds by shielding domains below the normalized spectral gap from cross-task interference.

\vspace*{-0.5\baselineskip}
\section{Merging Is Optimal, to Second Order, as a Closed-Form Stage of Distillation}
\label{sec:theory}
\vspace*{-0.5\baselineskip}

This section establishes the identification on which the paper rests: to second order, distillation and Fisher merging optimize the same objective (Section~\ref{sec:theory_mopd_fisher}); on the task-vector subspace that objective has a closed-form minimizer (Section~\ref{sec:theory_subspace_ta}); and, to first order in the tangent space, combining parameters adds implicit rewards, which yields a validity test (Section~\ref{sec:theory_homomorphism}). Proofs are in Appendix~\ref{app:proofs}.

\vspace*{-0.5\baselineskip}
\subsection{Setting}
\label{sec:theory_setting}
\vspace*{-0.5\baselineskip}

Let $W_{\mathrm{ref}} \in \mathbb{R}^P$ denote the base weights ($P$ parameters) and $\pi_\theta(\cdot \mid x)$ the autoregressive policy of a model with weights $\theta \in \mathbb{R}^P$ on prompt $x$, with $y$ a completion. We consider $K$ homologous specialists $\{W_{\phi_1}, \dots, W_{\phi_K}\}$, fine-tuned from the same $W_{\mathrm{ref}}$ on datasets $\mathcal{D}_d$ via RL on each domain, with policies $\pi_{W_{\phi_d}}$. The task vector of specialist $d$ is $\tau_d := W_{\phi_d} - W_{\mathrm{ref}}$, forming $\mathcal{T} := [\tau_1, \dots, \tau_K] \in \mathbb{R}^{P \times K}$.
Multi-teacher on-policy distillation (MOPD) trains a student $\theta$ by minimizing the domain-weighted reverse-KL to the specialist policies, with mixture weights $p_d$:
\begin{equation}
J_{\mathrm{MOPD}}(\theta) := \sum_{d=1}^K p_d \, \mathbb{E}_{x \sim \mathcal{D}_d} \left[ \mathrm{KL}\big( \pi_\theta(\cdot \mid x) \,\|\, \pi_{W_{\phi_d}}(\cdot \mid x) \big) \right], \quad p_d \ge 0, \quad \sum_{d=1}^K p_d = 1.
\label{eq:mopd_obj}
\end{equation}

\vspace*{-0.5\baselineskip}
\subsection{MOPD Is Fisher Merging to Second Order}
\label{sec:theory_mopd_fisher}
\vspace*{-0.5\baselineskip}

\begin{theorem}[Second-Order Equivalence of MOPD to Fisher Merging]
\label{thm:mopd_fisher}
Assume the specialists share a basin $\mathcal{B} \subset \mathbb{R}^P$ (a connected neighborhood of $W_{\mathrm{ref}}$ containing every $W_{\phi_d}$) on which each mapping $\theta \mapsto \mathrm{KL}(\pi_\theta \| \pi_{W_{\phi_d}})$ is three times differentiable with bounded third derivative. For $\theta \in \mathcal{B}$, the second-order Taylor expansion of $J_{\mathrm{MOPD}}(\theta)$ around the specialist checkpoints $\{W_{\phi_d}\}_{d=1}^K$ is:
\begin{equation}
J_{\mathrm{MOPD}}(\theta) = \frac{1}{2} \sum_{d=1}^K p_d \, (\theta - W_{\phi_d})^\top F_d (\theta - W_{\phi_d}) + \mathcal{O}\left( \max_d \|\theta - W_{\phi_d}\|^3 \right),
\label{eq:mopd_taylor}
\end{equation}
where $F_d := \mathbb{E}_{x \sim \mathcal{D}_d, y \sim \pi_{W_{\phi_d}}} [ \nabla_\theta \log \pi_{W_{\phi_d}}(y \mid x) \, \nabla_\theta \log \pi_{W_{\phi_d}}(y \mid x)^\top ] \in \mathbb{R}^{P \times P}$ is the empirical Fisher matrix of specialist $d$. The unconstrained global minimizer of Eq.~\ref{eq:mopd_taylor} is the Fisher merge:
\begin{equation}
\theta_{\mathrm{Fisher}}^\star = \left( \sum_{d=1}^K p_d F_d \right)^{-1} \sum_{d=1}^K p_d F_d W_{\phi_d}.
\label{eq:fisher_merge_global}
\end{equation}
\end{theorem}
The proof (Appendix~\ref{app:proof_thm1}) relies only on the fact that at $\theta = W_{\phi_d}$, the KL divergence and its gradient vanish while its Hessian equals $F_d$. Equation~\ref{eq:fisher_merge_global} recovers the Fisher-weighted average of \citet{matena2022merging}, originally derived via a Laplace approximation to the posterior. Theorem~\ref{thm:mopd_fisher} establishes that this exact point is also the second-order optimum of the multi-teacher reverse-KL objective. The two paradigms therefore differ in approximation order rather than optimization goal: while $J_{\mathrm{MOPD}}$ aligns student rollouts to teachers in distribution space, it corresponds in parameter space to the aggregate Fisher distance $\frac{1}{2}\sum_d p_d \|\theta - W_{\phi_d}\|_{F_d}^2$ with $\|v\|_{F_d}^2 := v^\top F_d v$. Merging is thus optimal, to second order, as a closed-form stage of distillation: gradient steps taken from the Fisher merge act only on the third-order remainder, so merging should come first. What a linear combination in the task-vector subspace leaves unrecovered is a per-domain error, which the residual program of Section~\ref{sec:method_phase1} measures.

\vspace*{-0.5\baselineskip}
\subsection{Subspace Fisher Projection and Task Arithmetic}
\label{sec:theory_subspace_ta}
\vspace*{-0.5\baselineskip}

Eq.~\ref{eq:fisher_merge_global} requires inverting a $P \times P$ Fisher matrix, which is intractable for modern LLMs. We therefore restrict the student to the $K$-dimensional task subspace $\theta(\boldsymbol{\alpha}) = W_{\mathrm{ref}} + \mathcal{T} \boldsymbol{\alpha}$, with one coefficient $\alpha_d$ per specialist. Let $\mathbf{e}_d$ be the $d$-th standard basis vector, so $\theta(\mathbf{e}_d)=W_{\phi_d}$ and $\theta(\boldsymbol{\alpha})-W_{\phi_d}=\mathcal{T}(\boldsymbol{\alpha}-\mathbf{e}_d)$. That substitution compresses Eq.~\ref{eq:mopd_taylor} into a $K \times K$ system:
\begin{equation}
J_{\mathrm{sub}}(\boldsymbol{\alpha}) := \frac{1}{2} \sum_{d=1}^K p_d \, (\boldsymbol{\alpha} - \mathbf{e}_d)^\top S_d (\boldsymbol{\alpha} - \mathbf{e}_d), \quad S_d := \mathcal{T}^\top F_d \mathcal{T} \in \mathbb{R}^{K \times K},
\label{eq:subspace_fisher_obj}
\end{equation}
where $S_d$ is the projected Fisher of specialist $d$. Setting $\nabla_{\boldsymbol{\alpha}} J_{\mathrm{sub}} = \mathbf{0}$ gives the closed-form subspace solution:
\begin{equation}
\boldsymbol{\alpha}_{\mathrm{sub}}^\star = \left( \sum_{d=1}^K p_d S_d \right)^{-1} \sum_{d=1}^K p_d S_d \mathbf{e}_d.
\label{eq:subspace_alpha}
\end{equation}
\begin{remark}
\label{rem:task_arith}
Task Arithmetic~\citep{ilharco2023editing} forms $W_{\mathrm{ref}} + \mathcal{T}\boldsymbol{\alpha}$ with hand-set coefficients, typically the isotropic $\boldsymbol{\alpha}=\alpha\mathbf{1}$ with $\alpha=1$. Eq.~\ref{eq:subspace_alpha} identifies the coefficients that make this family the exact metric projection of the quadratic Fisher objective onto the task subspace ($\mathcal{T}^\top \nabla_\theta J_{\mathrm{MOPD}}(\theta(\boldsymbol{\alpha}_{\mathrm{sub}}^\star)) = \mathbf{0}$); with these coefficients it is the computable closed-form instance of Fisher merging on LLMs.
\end{remark}

\vspace*{-0.5\baselineskip}
\subsection{First-Order Homomorphism and a Validity Diagnostic}
\label{sec:theory_homomorphism}
\vspace*{-0.5\baselineskip}

The subspace solution still depends on $F_d$. Forming $S_d=\mathcal{T}^\top F_d\mathcal{T}$ from parameters is impractical on LLMs, so the computations of Section~\ref{sec:method} run in the space of implicit rewards. The token-level implicit reward of specialist $d$ relative to $W_{\mathrm{ref}}$ is $r_d(x, y) := \log (\pi_{W_{\phi_d}}(y \mid x) / \pi_{W_{\mathrm{ref}}}(y \mid x))$, the log-probability ratio that RL fine-tuning implicitly optimizes~\citep{rafailov2023direct,bai2026multi}; write $\mathbf{r}(x, y) := [r_1(x,y),\dots,r_K(x,y)]^\top \in \mathbb{R}^K$ for the stacked reward. Extending tangent-space analyses of merging~\citep{ortizjimenez2023tangent}, a first-order expansion around $W_{\mathrm{ref}}$ yields the following homomorphism.

\begin{proposition}[First-Order Homomorphism]
\label{prop:homomorphism}
Let $f(x;\theta)$ be the logit vector at prompt-prefix $x$ and $Z(x;\theta)$ its softmax partition function. A first-order expansion around $W_{\mathrm{ref}}$ gives $r_d(x, y) \approx \nabla_\theta f(x; W_{\mathrm{ref}})[y]^\top \tau_d - \Delta \log Z_d(x)$, where $\Delta\log Z_d(x):=\log Z(x;W_{\phi_d})-\log Z(x;W_{\mathrm{ref}})$. For any linear combination $\theta(\boldsymbol{\alpha}) = W_{\mathrm{ref}} + \mathcal{T} \boldsymbol{\alpha}$:
\begin{equation}
\log \frac{\pi_{\theta(\boldsymbol{\alpha})}(y \mid x)}{\pi_{W_{\mathrm{ref}}}(y \mid x)} \approx \sum_{d=1}^K \alpha_d r_d(x, y) = \boldsymbol{\alpha}^\top \mathbf{r}(x, y).
\label{eq:homomorphism_formula}
\end{equation}
\end{proposition}

Proposition~\ref{prop:homomorphism} provides two operational capabilities. First, the projected score $\mathcal{T}^\top\nabla_\theta\log\pi_{W_{\phi_d}}$ equals $\mathbf{r}$ up to a mean shift (Appendix~\ref{app:proof_residuals}), so the projected Fisher is an empirical reward covariance, $S_d \approx \mathrm{Cov}_{\pi_{W_{\phi_d}}}(\mathbf{r})$, and the per-domain residual in Section~\ref{sec:method_phase1} can be computed from forward passes alone. Second, the log-ratio of the merged model, $\ell_{\boldsymbol{\alpha}}(x,y) := \log \frac{\pi_{\theta(\boldsymbol{\alpha})}(y \mid x)}{\pi_{W_{\mathrm{ref}}}(y \mid x)}$, should scale linearly with the superposed reward $\boldsymbol{\alpha}^\top \mathbf{r}(x, y)$. The unexplained variance when regressing $\ell_{\boldsymbol{\alpha}}$ on $\boldsymbol{\alpha}^\top \mathbf{r}$ over sampled tokens is then a label-free test of whether the first-order superposition of Proposition~\ref{prop:homomorphism} still holds at that checkpoint:
\begin{equation}
\mathrm{Err}_{\mathrm{lin}}(\theta(\boldsymbol{\alpha})) := 1 - R^2\big(\ell_{\boldsymbol{\alpha}} \sim \boldsymbol{\alpha}^\top \mathbf{r}\big) = \min_{a,b} \frac{\mathbb{E}\big[ ( \ell_{\boldsymbol{\alpha}} - a\,\boldsymbol{\alpha}^\top \mathbf{r} - b )^2 \big]}{\mathrm{Var}[\ell_{\boldsymbol{\alpha}}]}.
\label{eq:err_lin_def}
\end{equation}
We treat $\mathrm{Err}_{\mathrm{lin}}(\theta(\boldsymbol{\alpha})) \le \delta_{\mathrm{tol}}$ under a fixed threshold $\delta_{\mathrm{tol}}$ as the operational criterion for a valid linear basin (we use \emph{linear basin} for this regime throughout). Large $\mathrm{Err}_{\mathrm{lin}}$ means the first-order fit has broken; Section~\ref{sec:method_phase2} uses the same test to cap how far a Phase~2 write-back may move.

\vspace*{-0.5\baselineskip}
\section{RGBM: Residual-Governed Branch-and-Merge}
\label{sec:method}
\vspace*{-0.5\baselineskip}

Section~\ref{sec:theory} establishes that parameter merging should precede distillation, and that whatever the linear combination leaves unrecovered is what a later stage must address. \textbf{Residual-Governed Branch-and-Merge} (\textsc{RGBM}; Figure~\ref{fig:rgbm_pipeline}; full procedure in Algorithm~\ref{alg:rgbm}, Appendix~\ref{app:algorithm_systems}) operationalizes this insight into two distinct phases. Phase~1 (Section~\ref{sec:method_phase1}) solves a convex subspace merge and then normalizes its per-domain remainders into the \emph{domain residual spectrum}. After analyzing why continuing joint distillation degrades domains with low normalized residuals (Section~\ref{sec:method_whynot}), Phase~2 (Section~\ref{sec:method_phase2}) trains isolated branches exclusively for domains marked by the spectrum as unrecovered, writing them back via calibrated shrinkage.

\begin{figure}[t]
\centering
\includegraphics[width=0.78\textwidth]{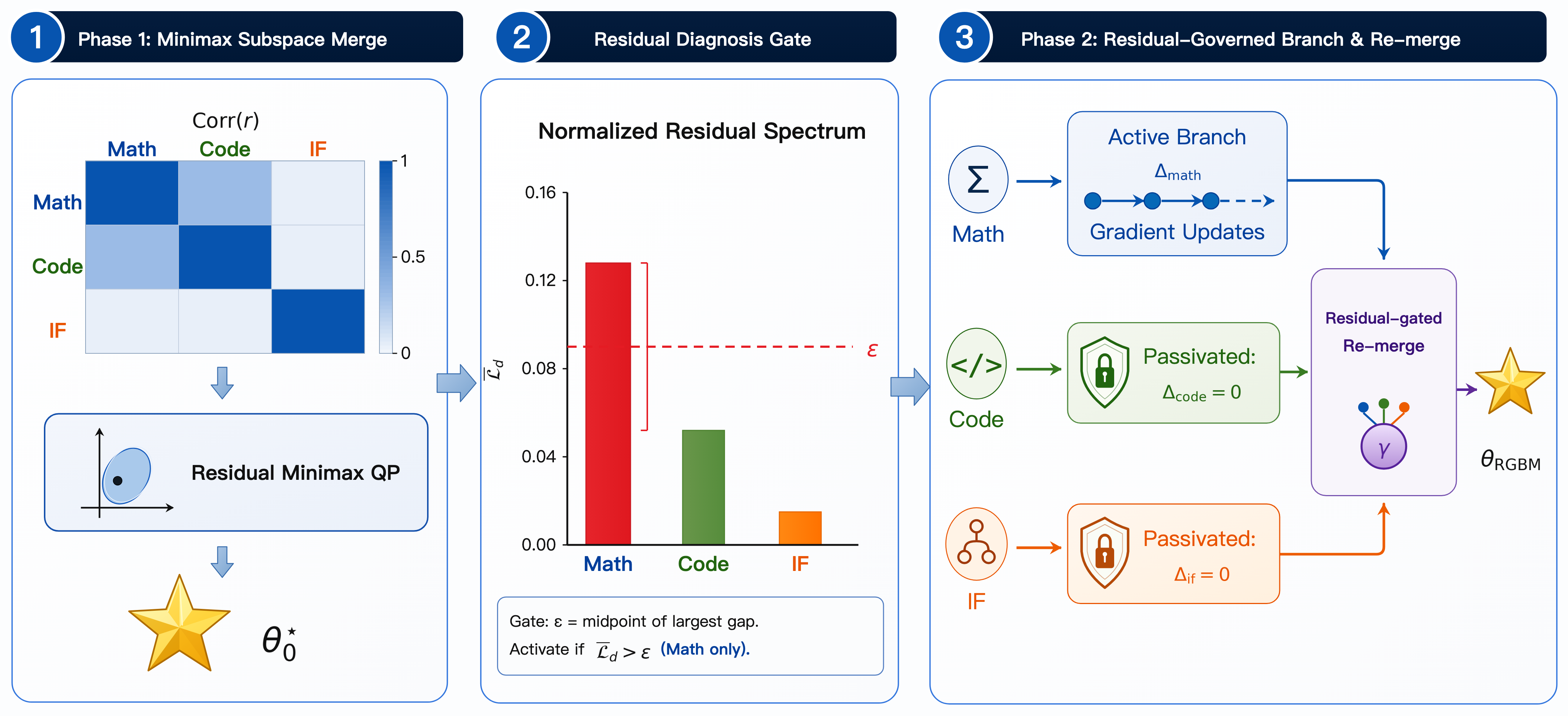}
\caption{\textbf{Two-phase RGBM architecture.} (1)~Phase~1 scores implicit rewards on unlabelled prompts, forms per-domain second-moment matrices $M_d$, and minimizes the worst residual $\mathcal{L}_d$ to obtain $\theta_0^\star$. The heatmap is the reward correlation, not $M_d$ itself. (2)~The bars are the energy-normalized Open-MOPD spectrum $\overline{\boldsymbol{\mathcal{L}}}$; the gate cuts at the midpoint of its largest adjacent gap (threshold $\epsilon$; Eq.~\ref{eq:epsilon_midgap}), so only Math is active. (3)~Phase~2 writes active updates back with shrinkage $\gamma$ to obtain $\theta_{\mathrm{RGBM}}$; if every domain is passivated, $\theta_{\mathrm{RGBM}}=\theta_0^\star$. See Sections~\ref{sec:method_phase1} and~\ref{sec:method_phase2} for details.}
\label{fig:rgbm_pipeline}
\end{figure}

\vspace*{-0.5\baselineskip}
\subsection{Phase 1: Minimax Subspace Merge and the Domain Residual Spectrum}
\label{sec:method_phase1}
\vspace*{-0.5\baselineskip}

\paragraph{Per-domain residual from unlabelled prompts.} Under the tangent-space homomorphism of Proposition~\ref{prop:homomorphism} the projected score equals $\mathbf{r}$ up to a mean shift, so $S_d \approx \mathrm{Cov}_{\pi_{W_{\phi_d}}}(\mathbf{r})$ can be estimated without gradients. Let $\mathcal{D}_{\mathrm{calib},d}$ be unlabelled prompts from domain $d$. Rather than center the rewards, we use the empirical second-moment $M_d := \mathbb{E}_{x\sim\mathcal{D}_{\mathrm{calib},d},\, y\sim\pi_{W_{\phi_d}}}[\mathbf{r}\mathbf{r}^\top] \in \mathbb{R}^{K \times K}$, assembled from forward passes of the base and the $K$ specialists on completions sampled from specialist $d$. $M_d$ equals that covariance plus a mean-shift outer product, and therefore upper-bounds the projected Fisher residual (Appendix~\ref{app:proof_residuals}). We define the per-domain residual and its energy-normalized counterpart as
\begin{equation}
\begin{aligned}
\mathcal{L}_d(\boldsymbol{\alpha})
&:= (\boldsymbol{\alpha} - \mathbf{e}_d)^\top M_d (\boldsymbol{\alpha} - \mathbf{e}_d), \\
\overline{\mathcal{L}}_d(\boldsymbol{\alpha})
&:= \frac{\mathcal{L}_d(\boldsymbol{\alpha})}{\mathbb{E}[r_d^2]}.
\end{aligned}
\label{eq:normalized_residual}
\end{equation}
Phase~1 optimizes $\mathcal{L}_d$. After the merge is fixed, division by $\mathbb{E}[r_d^2]$, the second moment of the specialist's own reward (its \emph{reward energy}), expresses the remainder as a fraction of domain-$d$ energy. Thus $\overline{\mathcal{L}}_d$ is comparable across domains and is used for diagnosis and gating, but not for solving the merge. Because $M_d$ contains the cross-domain products $\mathbb{E}[r_j r_k]$, $\overline{\mathcal{L}}_d(\boldsymbol{\alpha}^\star)$ is the fraction of domain-$d$ reward energy still unexplained after the coefficients have traded the domains off against one another. A value near zero means this joint combination already accounts for domain $d$. It does not mean an isolated branch on $d$ has nothing left to fit. A large value means a substantial share of that energy remains once the other domains have been accommodated. Section~\ref{sec:exp_objective} measures the merges obtained when the same programs are solved on $\overline{\mathcal{L}}_d$.

\paragraph{Minimax merge.} Eq.~\ref{eq:subspace_alpha} requires domain weights $p_d$, which are rarely known and trade one domain against another, so the Fisher stage of Section~\ref{sec:theory_subspace_ta} is not directly computable. We remove this degree of freedom by taking the worst case over all admissible weights, the probability simplex $\Delta_K$: since $\max_{\mathbf{p} \in \Delta_K} \sum_d p_d \mathcal{L}_d(\boldsymbol{\alpha}) = \max_d \mathcal{L}_d(\boldsymbol{\alpha})$, the problem becomes the convex Minimax quadratic program (QP)
\begin{equation}
\boldsymbol{\alpha}^\star = \arg\min_{\boldsymbol{\alpha} \in \mathbb{R}^K} \max_{d \in \{1, \dots, K\}} \mathcal{L}_d(\boldsymbol{\alpha}),
\label{eq:minimax_qp}
\end{equation}
i.e., the worst subspace residual under adversarially chosen domain weights. Its solution gives the zero-training checkpoint $\theta_0^\star := W_{\mathrm{ref}} + \mathcal{T} \boldsymbol{\alpha}^\star$. We then report the \emph{Domain Residual Spectrum} $\overline{\boldsymbol{\mathcal{L}}} := [\overline{\mathcal{L}}_1(\boldsymbol{\alpha}^\star), \dots, \overline{\mathcal{L}}_K(\boldsymbol{\alpha}^\star)]$ for cross-domain comparison and Phase~2 allocation. The QP is $K$-dimensional, so its cost is dominated by assembling the $M_d$ (Appendix~\ref{app:md_cost_analysis}). For numerical stability, the QP replaces $M_d$ by $\tilde{M}_d=M_d+\lambda I$; the reported spectrum uses the unregularized $M_d$. We leave $\boldsymbol{\alpha}$ unconstrained, so a coefficient may lie outside $[0,1]$ if that reduces the worst residual.

\vspace*{-0.5\baselineskip}
\subsection{Why Not Simply Keep Distilling from \texorpdfstring{$\theta_0^\star$}{the Merge}?}
\label{sec:method_whynot}
\vspace*{-0.5\baselineskip}

The quadratic model of Theorem~\ref{thm:mopd_fisher} is stationary at the Fisher merge, and approximately stationary on the subspace at $\theta(\boldsymbol{\alpha}_{\mathrm{sub}}^\star)$. Phase~1's $\theta_0^\star$ is not that point: it minimizes the worst residual, so $\nabla J_{\mathrm{MOPD}}(\theta_0^\star)$ need not vanish, and continued joint distillation is not a refinement of a solved linear stage. Two mechanisms outside the quadratic regime then make it harmful. Each domain divergence $f_d(\theta) := \mathbb{E}_{x\sim\mathcal{D}_d}[\mathrm{KL}(\pi_\theta\,\|\,\pi_{W_{\phi_d}})]$ has $\nabla_\theta f_d(\theta_0^\star) \neq \mathbf{0}$, so mini-batch steps on shared parameters pull in conflicting directions, including components off the task subspace. A near-zero normalized residual $\overline{\mathcal{L}}_d$ adds a second problem, and it is a joint one. An isolated branch can still reduce domain $d$'s own loss, because $\overline{\mathcal{L}}_d$ is the gap left after the cross-domain tradeoff, not a claim that domain $d$ is saturated. Writing that branch into shared weights moves every domain: the step is spent where the joint merge has already closed the gap, so the three-domain loss can worsen even while domain $d$ itself improves. Joint distillation from $\theta_0^\star$ should therefore degrade those low-residual domains on the shared model. A token-level KL penalty to $\pi_{\theta_0^\star}$ cannot prevent this, because the collision is in the weights. Section~\ref{sec:exp_joint_failure} checks both predictions.

\vspace*{-0.5\baselineskip}
\subsection{Phase 2: Residual Gate, Isolated Branches, and Shrinkage Re-merge}
\label{sec:method_phase2}
\label{sec:method_remerging}
\vspace*{-0.5\baselineskip}

RGBM therefore never applies a joint gradient to the merge; it partitions domains by the normalized spectrum $\overline{\boldsymbol{\mathcal{L}}}$ and treats the two groups differently.

\paragraph{Gate.} Sort $\overline{\mathcal{L}}_{(1)}\ge\cdots\ge\overline{\mathcal{L}}_{(K)}$, append a floor $\overline{\mathcal{L}}_{(K+1)}:=0$, and cut at the midpoint of the largest adjacent gap,
\begin{equation}
k^\star = \arg\max_{i=1,\dots,K}\bigl(\overline{\mathcal{L}}_{(i)}-\overline{\mathcal{L}}_{(i+1)}\bigr),
\qquad
\epsilon := \tfrac12\bigl(\overline{\mathcal{L}}_{(k^\star)}+\overline{\mathcal{L}}_{(k^\star+1)}\bigr).
\label{eq:epsilon_midgap}
\end{equation}
Domain $d$ is \emph{active} iff $\overline{\mathcal{L}}_d>\epsilon$, otherwise \emph{passivated}; the floor makes every domain active when normalized residuals cluster, and any $\epsilon$ inside the same gap yields the same allocation (Appendix~\ref{app:epsilon_sensitivity}).

\paragraph{Passivated and active domains.} Passivated domains ($\overline{\mathcal{L}}_d \le \epsilon$) receive no shared write-back ($\Delta_d=\mathbf{0}$). The joint merge already accounts for them, which does not preclude an isolated gain on that domain alone. Each active domain trains an isolated branch $\theta_d \leftarrow \theta_0^\star$ on its native data $\mathcal{D}_d$ with no gradient sharing, giving an update $\Delta_d = \theta_d^{(\mathrm{end})} - \theta_0^\star$; isolation guarantees that no gradient computed on domain $d$ touches a passivated domain during training. Whether the update can be written back without damaging them is decided next.

\paragraph{Shrinkage re-merge.} The sole Phase~2 write-back is a shrunken sum of the active updates,
\begin{equation}
\theta_{\mathrm{RGBM}} = \theta_0^\star + \sum_{d: \overline{\mathcal{L}}_d > \epsilon} \gamma_d^\star \, \Delta_d,
\qquad
\gamma_d^\star = \max \bigl\{\gamma \in \Gamma \mid \mathrm{Err}_{\mathrm{lin}}(\theta_0^\star + \gamma \Delta_d) \le \delta_{\mathrm{tol}}\bigr\},
\label{eq:rgbm_remerge_eq}
\end{equation}
where $\Gamma\subset(0,1]$ is a fixed discrete grid and $\mathrm{Err}_{\mathrm{lin}}$ is evaluated on the same unlabelled calibration prompts $\mathcal{D}_{\mathrm{calib}}=\bigcup_d\mathcal{D}_{\mathrm{calib},d}$ used for $M_d$.
$\gamma_d^\star$ is the largest grid value at which that single write-back still passes the linearity test of Section~\ref{sec:theory_homomorphism}, so the write-back keeps the merged model's log-ratio within $\delta_{\mathrm{tol}}$ of the Phase~1 superposition $\boldsymbol{\alpha}^{\star\top}\mathbf{r}$, which is held fixed as the predictor; Proposition~\ref{prop:local_basin_invariance} (Appendix~\ref{app:proof_remerge_basin}) bounds the first-order remainder at $\theta_0^\star+\gamma\Delta_d$ by a quantity that is non-decreasing in $\gamma$. When several branches are active, evaluation uses their shrunken sum. If the gate leaves no active domain, $\theta_{\mathrm{RGBM}}=\theta_0^\star$. The whole procedure has three fixed constants ($\lambda$, $\Gamma$, $\delta_{\mathrm{tol}}$) and no labelled data; their values and the resulting allocations on both suites are given in Section~\ref{sec:exp_setup}.

\vspace*{-1\baselineskip}
\section{Experiments}
\label{sec:experiments}

\vspace*{-0.5\baselineskip}
\subsection{Setup}
\vspace*{-0.5\baselineskip}
\label{sec:exp_setup}

\paragraph{Main suite (Open-MOPD).} Our primary testbed is the Open-MOPD specialist suite~\citep{openmopd2026}: SmolLM3-3B~\citep{smollm3_2025} with $K=3$ homologous RL specialists (the usual MOPD domain count;~\citealp{openmopd2026,d3mopd2026,li2025policyfiltration}), fine-tuned from the mixed-SFT checkpoint MixSFT as $W_{\mathrm{ref}}$. The specialists are Math, Code, and Instruction Following (IF), scored on AIME 2024/2025, LiveCodeBench v5/v6~\citep{jain2024livecodebench}, and IFEval/IFBench~\citep{zhou2023instruction,pyatkin2025ifbench}. We report $\text{Macro} = \frac{1}{3}(\text{Math} + \text{Code} + \text{IF})$, where IF is the mean of IFEval and IFBench, and re-evaluate every system under one protocol. The italic RouteRL row is the routing oracle of \citet{openmopd2026} that serves all $K$ specialists; Multi-Teacher OPD is their released $600$-step student. Every Phase~2 system starts from the same Phase~1 merge $\theta_0^\star$. TIES~\citep{yadav2023ties} and DARE~\citep{yu2024dare} sparsify and sum the isolated on-policy distillation (OPD) branch updates $\Delta_d$---not the specialist task vectors---and RGBM uses those same branches, differing only in what is written back (Appendix Table~\ref{tab:systems}). Per-benchmark numbers and hyperparameters are in Appendices~\ref{app:stage1_full_tables} and \ref{app:training_params}.

\paragraph{Secondary suite (G-OPD).} To test transfer across model families we apply the same procedure, unchanged, to the G-OPD suite~\citep{gopd2026}: Qwen3-4B~\citep{qwen3_2025} with Math and Code specialists (intermediate checkpoints of Table~2 in \citealp{gopd2026}; 500/300 steps) on their 7 official benchmarks. Supervised fine-tuning (SFT), OPD, and the extended recipe ExOPD are from \citet{gopd2026}; merges are evaluated by us under the official protocol. The italic Domain-Matched Teacher row is that specialist pair, an oracle rather than a single generalist. Per-benchmark results and the fully trained pair are in Appendix~\ref{app:gopd_benchmarks}.

\paragraph{RGBM constants and resulting allocations.} We use $\lambda=10^{-4}$, $\Gamma=\{0.125,0.25,0.5,0.75\}$, and $\delta_{\mathrm{tol}}=0.10$ on both suites. On Open-MOPD the energy-normalized spectrum is $\overline{\boldsymbol{\mathcal{L}}}=(0.128,0.052,0.015)$ for (Math, Code, IF); the largest gap is between Math and Code, so $\epsilon=0.090$, only Math is active, and Eq.~\ref{eq:rgbm_remerge_eq} selects $\gamma^\star_{\mathrm{math}}=0.125$, the largest $\gamma\in\Gamma$ with $\mathrm{Err}_{\mathrm{lin}}\le\delta_{\mathrm{tol}}$ ($\mathrm{Err}_{\mathrm{lin}}=0.0701$; Appendix Table~\ref{tab:p1_ablation_app}). On G-OPD, $\boldsymbol{\alpha}^\star=(0.675,0.666)$ and $\overline{\boldsymbol{\mathcal{L}}}=(0.094,0.090)$ cluster, so the largest gap is down to the floor, $\epsilon=0.045$, both domains are active, and the same rule gives $\gamma^\star=0.5$. Hereafter, low- and high-residual domains refer to this normalized spectrum. Assembling the $M_d$ costs $<15$ minutes on 8 GPUs ($<2$ GPU$\cdot$h) of unlabelled scoring and the QP solves in $<10$ ms.

\vspace*{-0.5\baselineskip}
\subsection{Merge Before You Distill: Zero-Training Merge vs.\ Trained Distillation}
\label{sec:exp_main_results}
\vspace*{-0.5\baselineskip}

\begin{wraptable}{r}{0.50\textwidth}
\vspace*{-0.75\baselineskip}
\centering
\caption{\textbf{Open-MOPD} (SmolLM3-3B; accuracy). Macro: mean of Math, Code, and IF. Bold: best deployable model. Italic: routing oracle. Gated Math: admitted branch only, at the listed $\gamma$.}
\label{tab:master_results}
\scriptsize
\renewcommand{\arraystretch}{0.85}\setlength{\tabcolsep}{2.5pt}
\begin{tabular}{@{}lcccc@{}}
\toprule
\textbf{Method} & \textbf{Math} & \textbf{Code} & \textbf{IF} & \textbf{Macro} \\
\midrule
RouteRL (Oracle Ensemble) & \textit{27.24} & \textit{26.27} & \textit{53.52} & \textit{35.68} \\
Multi-Teacher OPD & 24.53 & 21.18 & 47.12 & 30.93 \\
\midrule
Task Arithmetic & 24.35 & 22.81 & 46.77 & 31.31 \\
Subspace Fisher Merge & 24.11 & 23.59 & 47.99 & 31.90 \\
\textbf{RGBM Phase 1 (Ours, $\theta_0^\star$)} & 23.86 & 24.09 & 49.60 & 32.52 \\
\midrule
Uniform Distillation (from $\theta_0^\star$) & 23.97 & 21.36 & 47.42 & 30.91 \\
TIES (all branches) & 24.61 & 21.56 & 46.72 & 30.96 \\
DARE (all branches) & 24.58 & 22.40 & 47.28 & 31.42 \\
\midrule
Gated Math ($\gamma=1.0$) & 25.03 & 22.51 & 47.87 & 31.80 \\
Gated Math ($\gamma=0.5$) & 25.42 & 23.48 & 48.08 & 32.32 \\
\textbf{RGBM Phase 2 (Ours)} & 23.93 & 24.63 & 50.52 & 33.03 \\
\bottomrule
\end{tabular}
\vspace*{-0.5\baselineskip}
\end{wraptable}
\paragraph{The closed-form merge matches or beats trained distillation.} On Open-MOPD (Table~\ref{tab:master_results}) the Minimax merge $\theta_0^\star$ reaches $32.52$ Macro at zero gradient cost, above Multi-Teacher OPD, the released $600$-step student ($30.93$; $144$ GPU$\cdot$h), and above isotropic Task Arithmetic ($31.31$). On G-OPD (Table~\ref{tab:gopd_main_results}) the same procedure gives $53.34$, above SFT, OPD, and Task Arithmetic and within $0.45$ of the trained ExOPD recipe (equal-weight averaging collapses to $36.75$: the gain is in the coefficients). The $31.90$ row is the subspace Fisher solution of Eq.~\ref{eq:subspace_alpha} under uniform weights $p_d=1/K$, i.e., the minimizer of $\sum_d p_d\,\mathcal{L}_d(\boldsymbol{\alpha})$ with the Phase~1 second moments $M_d$ in place of $S_d$. Eq.~\ref{eq:subspace_alpha} requires the domain weights $p_d$, and those weights are not known; the uniform choice is a default, not an estimate. Minimax, which minimizes the worst residual instead, is $0.62$ higher and lifts IF from $47.99$ to $49.60$---when no domain may be sacrificed, the adversarial weighting is the operational choice.

\paragraph{RGBM improves on the merge on both suites.} Writing back only the gated Math branch raises Open-MOPD to $33.03$ (+0.51 over $\theta_0^\star$) with zero extra joint-training steps; on G-OPD both domains are active, and writing back both isolated branches gives $54.14$, above ExOPD.

\vspace*{-0.5\baselineskip}
\subsection{The Quadratic That Is Minimized}
\label{sec:exp_objective}
\vspace*{-0.5\baselineskip}

Phase~1 minimizes $\mathcal{L}_d$. This quadratic upper-bounds the domain term $(\boldsymbol{\alpha}-\mathbf{e}_d)^\top S_d(\boldsymbol{\alpha}-\mathbf{e}_d)$ of $J_{\mathrm{sub}}$ (Eq.~\ref{eq:subspace_fisher_obj}; Appendix~\ref{app:proof_residuals}), since $M_d$ equals $S_d\approx\mathrm{Cov}(\mathbf{r})$ plus a mean-shift outer product. The uniform row above is that bound, with $M_d$ in place of $S_d$. The factor $1/\mathbb{E}[r_d^2]$ is applied after $\boldsymbol{\alpha}$ is chosen, which rescales the reported spectrum and leaves the minimizer in place. Solving the same programs on $M_d/\mathbb{E}[r_d^2]$ weights domain $d$ by $p_d/\mathbb{E}[r_d^2]$. On the Open-MOPD calibration the energies are $\mathbb{E}[r_d^2]=(0.180,0.692,2.40)$ for (Math, Code, IF), so that program down-weights Code relative to Math by a factor of $3.8$. Task arithmetic already treats the coefficient as a control on accuracy and picks it on held-out labels~\citep{ilharco2023editing}. The solve below uses no labels.

\begin{table}[t]
\centering
\captionsetup{font=footnotesize,skip=2pt}
\caption{\textbf{Open-MOPD coefficient grid.} A minimizes the mean residual (uniform $p_d$, Eq.~\ref{eq:subspace_alpha}); C minimizes the worst residual (Eq.~\ref{eq:minimax_qp}). Scalar shares one coefficient; vector leaves $\boldsymbol{\alpha}$ free. Scores for the four $\mathcal{L}$ rows and for $\boldsymbol{\alpha}=\mathbf{1}/3$ are Table~\ref{tab:stage1_full_breakdown}; the four $\overline{\mathcal{L}}$ rows use the same protocol. Coefficients are Math/Code/IF.}
\label{tab:s1_objective}
\footnotesize
\renewcommand{\arraystretch}{0.85}\setlength{\tabcolsep}{2.5pt}
\begin{tabular}{@{}llcccc@{}}
\toprule
\textbf{Program} & $\boldsymbol{\alpha}$ & \textbf{Math} & \textbf{Code} & \textbf{IF} & \textbf{Macro} \\
\midrule
$\boldsymbol{\alpha}=\mathbf{1}/3$ & 0.333/0.333/0.333 & 20.81 & 19.15 & 44.77 & 28.24 \\
\midrule
A scalar $\mathcal{L}$ & 0.626/0.626/0.626 & 24.97 & 23.28 & 47.64 & 31.96 \\
A vector $\mathcal{L}$ & 0.527/0.609/0.821 & 24.11 & 23.59 & 47.99 & 31.90 \\
C scalar $\mathcal{L}$ & 0.579/0.579/0.579 & 24.97 & 23.49 & 45.10 & 31.19 \\
C vector $\mathcal{L}$ ($\theta_0^\star$) & 0.597/0.620/0.800 & 23.86 & 24.09 & 49.60 & \textbf{32.52} \\
\midrule
A scalar $\overline{\mathcal{L}}$ & 0.501/0.501/0.501 & 24.38 & 23.36 & 44.86 & 30.86 \\
A vector $\overline{\mathcal{L}}$ & 0.840/0.260/0.841 & 24.19 & 19.57 & 48.00 & 30.59 \\
C scalar $\overline{\mathcal{L}}$ & 0.477/0.477/0.477 & 21.41 & 19.76 & 45.79 & 28.99 \\
C vector $\overline{\mathcal{L}}$ & 0.724/0.426/0.630 & 24.22 & 20.73 & 47.02 & 30.66 \\
\bottomrule
\end{tabular}
\end{table}

Every program solved on $\mathcal{L}$ lands in Macro $31.19\sim 32.52$. Every program solved on $\overline{\mathcal{L}}$ lands in $28.99\sim 30.86$, so the best normalized merge ($30.86$) sits below the worst raw merge ($31.19$). The reweighting shows up on Code. The uniform vector on $\mathcal{L}$ keeps $\alpha_{\mathrm{code}}=0.609$ and Code at $23.59$; the same program on $\overline{\mathcal{L}}$ sets $\alpha_{\mathrm{code}}=0.260$ and Code falls to $19.57$. The published minimax vector keeps $\alpha_{\mathrm{code}}=0.620$ and Code at $24.09$. Equal-weight $\boldsymbol{\alpha}=\mathbf{1}/3$, which solves neither program, is lower still ($28.24$).

\begin{figure*}[t]
\centering
\includegraphics[width=0.98\textwidth]{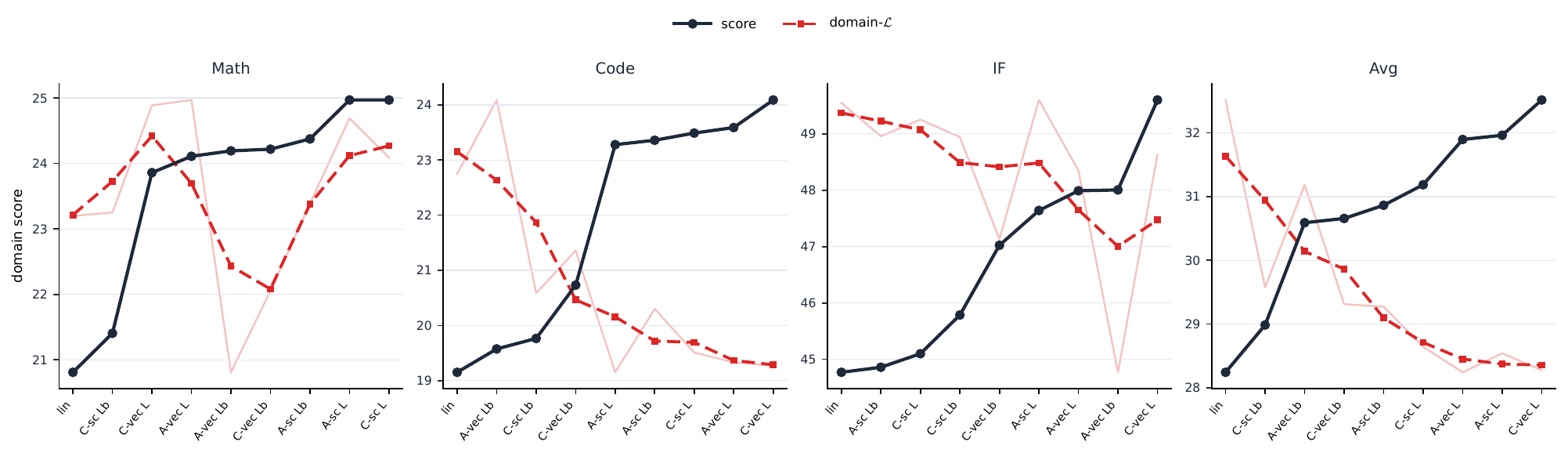}
\caption{\textbf{Domain score against the raw residual (Open-MOPD).} Each panel orders the nine merges of Table~\ref{tab:s1_objective} by that panel's score, so the score curve rises by construction. Domain-$\mathcal{L}$ is min-max scaled onto that score range. A/C: uniform and minimax. sc/vec: shared scalar and free vector. L / Lb: the program minimizes $\mathcal{L}$ or $\mathcal{L}_d/\mathbb{E}[r_d^2]$. lin: $\boldsymbol{\alpha}=\mathbf{1}/3$. The energy-normalized spectrum is used only for the Phase~2 gate in Section~\ref{sec:exp_setup}.}
\label{fig:s1_objective}
\end{figure*}

With $M_d$ held fixed, $\mathcal{L}_d(\boldsymbol{\alpha})$ is a function of the coefficients alone. Figure~\ref{fig:s1_objective} orders the nine merges by score and overlays that raw residual, min-max scaled onto the score range. On Code, $\alpha_{\mathrm{code}}\ge 0.50$ keeps Code at $23.28\sim 24.09$ and $\alpha_{\mathrm{code}}\le 0.48$ puts it at $19.15\sim 20.73$ (Pearson $+0.89$ between $\alpha_{\mathrm{code}}$ and Code). On IF, $\alpha_{\mathrm{if}}$ and IF have Pearson $+0.87$, from $44.77$ at $\alpha_{\mathrm{if}}=0.333$ to $49.60$ at the published $\alpha_{\mathrm{if}}=0.800$. On Math, $\alpha_{\mathrm{math}}$ separates the floor: the two vectors with $\alpha_{\mathrm{math}}<0.5$ score $20.81$ and $21.41$, and every vector with $\alpha_{\mathrm{math}}\ge 0.5$ stays in $23.86\sim 24.97$. The coefficients track Code and IF, and they track whether Math has cleared that floor.

\vspace*{-0.5\baselineskip}
\subsection{Why the Write-Back Must Be Gated: Joint Distillation from the Merge}
\label{sec:exp_joint_failure}
\label{sec:analysis_joint_failure}
\vspace*{-0.5\baselineskip}

\begin{figure}[t]
\centering
\begin{subfigure}[b]{0.42\textwidth}
\centering
\includegraphics[width=\linewidth]{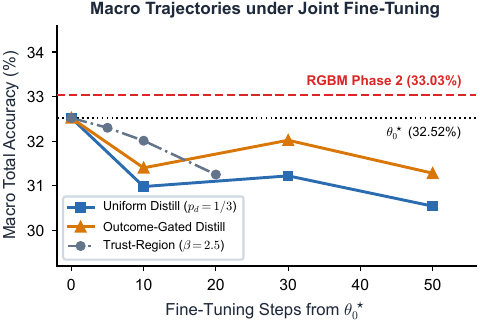}
\caption{Macro trajectory}
\label{fig:traj_macro}
\end{subfigure}
\hfill
\begin{subfigure}[b]{0.42\textwidth}
\centering
\includegraphics[width=\linewidth]{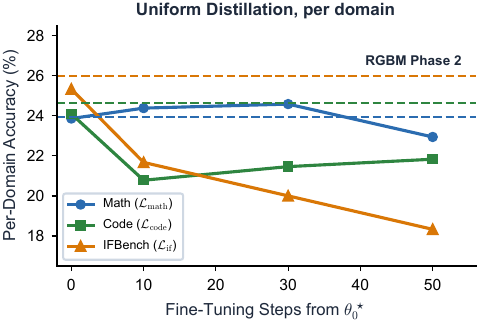}
\caption{Uniform Distillation, per domain}
\label{fig:traj_domain}
\end{subfigure}
\caption{\textbf{Joint distillation from $\theta_0^\star$ (Open-MOPD, $\le 50$ steps).} (a)~Macro under three joint objectives; dashed and dotted lines are RGBM Phase~2 and $\theta_0^\star$. (b)~Uniform Distillation: Math stays near the merge before a late drop; Code and IFBench fall. Dashed lines mark RGBM Phase~2 on those axes.}
\vspace*{-2\baselineskip}
\label{fig:trajectories}
\end{figure}

The obvious second stage is to keep distilling $\theta_0^\star$ jointly on all domains; Section~\ref{sec:method_whynot} argued that this should erode the low-residual domains. Figure~\ref{fig:trajectories} tracks three joint objectives initialized from $\theta_0^\star$: \emph{Uniform} (equal-weight reverse-KL), \emph{Outcome-Gated} (matching the teacher only on failed rollouts), and \emph{Trust-Region} (outcome-gated reverse-KL with a KL penalty $\beta=2.5$ toward $\pi_{\theta_0^\star}$). Appendix~\ref{app:stage2_scatter} extends this analysis to five distinct objectives.
(1)~\textbf{Low-residual domains deteriorate.} At the last listed checkpoint of each recipe trained to step $20$ or $50$, Code drops by $1.1\sim 2.9$ points below $\theta_0^\star$, IF falls behind the merge, and Macro stays in $30.5\sim 32.3$. Joint distillation therefore undoes what Phase~1 had already recovered.
(2)~\textbf{High-residual Math improves only marginally.} Uniform and Outcome-Gated trajectory means are $23.97$ and $24.69$ (Appendix~\ref{app:stage2_scatter}), against $23.86$ at $\theta_0^\star$. Continued joint steps do not convert the Math residual into a gain that offsets the drop on Code and IF.
(3)~\textbf{Interference is parameter-level.} Even when the Code reverse-KL is held below $10^{-3}$ nat/token, Code still degrades. The student can stay close to the Code teacher in distribution space, so the damage is coming from shared-weight updates, not from a Code-teacher mismatch.
(4)~\textbf{A token-level KL penalty cannot prevent collapse.} The Trust-Region anchor bounds estimated policy drift to $\widehat{\mathrm{KL}}(\pi_\theta\|\pi_{\theta_0^\star})\approx 1.4\times 10^{-3}$ nat/token, yet Macro decays to $31.25$ by step 20 (Code drops $24.09\to 21.18$). A scalar KL only restricts isotropic divergence; it does not block interference along the anisotropic directions captured by $\overline{\mathcal{L}}_d$.

\vspace*{-1\baselineskip}
\subsection{The Residual Spectrum Identifies Which Updates Can Be Written Back Safely}
\label{sec:exp_safe_writeback}
\vspace*{-0.5\baselineskip}

\begin{wraptable}{r}{0.50\textwidth}
\vspace*{-0.75\baselineskip}
\centering
\caption{\textbf{G-OPD} (Qwen3-4B). Math / Code: mean of the 4 math / 3 code tasks. Avg: 7-task mean. Italic: teacher oracle.}
\label{tab:gopd_main_results}
\footnotesize
\renewcommand{\arraystretch}{0.75}\setlength{\tabcolsep}{3.5pt}
\begin{tabular}{@{}lccc@{}}
\toprule
\textbf{Method} & \textbf{Math} & \textbf{Code} & \textbf{Avg} \\
\midrule
Domain-Matched Teacher & \textit{46.00} & \textit{61.17} & \textit{52.50} \\
\midrule
SFT & 44.33 & 60.80 & 51.39 \\
OPD & 46.38 & 60.57 & 52.46 \\
ExOPD & \textbf{47.65} & 61.97 & 53.79 \\
\midrule
Task Arithmetic & 43.05 & 60.65 & 50.59 \\
Equal-Weight Average & 22.55 & 55.67 & 36.75 \\
\textbf{RGBM Phase 1 (Ours)} & 46.17 & 62.89 & 53.34 \\
\midrule
\textbf{RGBM Phase 2 (Ours)} & 47.55 & \textbf{62.93} & \textbf{54.14} \\
\bottomrule
\end{tabular}
\vspace*{-0.5\baselineskip}
\end{wraptable}
The normalized spectrum $\overline{\boldsymbol{\mathcal{L}}}=(0.128,0.052,0.015)$ is computed from the Phase~1 solution before any Phase~2 update, and the gate it induces admits only the Math branch. We test this decision against the alternatives it rejects. Writing back the admitted Math update ($\gamma^\star=0.125$) raises Macro to $33.03$. Writing back the rejected IF update at a mid-grid factor ($\theta_0^\star+0.5\,\Delta_{\mathrm{if}}$; IF has no $\gamma^\star_{\mathrm{if}}$) lowers Macro to $31.10$ and Code to $20.70$ (Appendix~\ref{app:epsilon_sensitivity}). Writing back all three, as un-gated TIES and DARE do, gives $30.96$ and $31.42$, both below the merge. On G-OPD the normalized spectrum clusters, the gate keeps both domains, and writing back both raises the average by $0.80$.

We stress what this does and does not show. Isolated branches raise their own domain under all three objectives (Appendix Table~\ref{tab:residual_prediction}): IF by $+0.48\sim+1.63$, Code by $+0.43\sim+1.31$, and Math by $+0.58\sim+1.83$. A low residual therefore does not mean a domain has nothing left to learn on its own: $\overline{\mathcal{L}}_d$ is the coupled gap at $\boldsymbol{\alpha}^\star$, after all domains have been traded off together. A write-back is one shared update, so it moves domains other than the branch that was trained, and it is scored on Math, Code, and IF together. On that joint score the rejected IF write-back lowers Macro, while the admitted Math write-back raises it only at $\gamma^\star=0.125$. The spectrum chooses which branch may enter the shared model. The shrinkage grid chooses how far that branch may move it.

\vspace*{-0.5\baselineskip}
\subsection{Ablations: Gate, Isolation, and Shrinkage}
\label{sec:exp_ablation}
\vspace*{-0.5\baselineskip}

The lower half of Table~\ref{tab:master_results} isolates each RGBM component (full $\gamma$ grid in Appendix~\ref{app:gamma_sweep}). Uniform Distillation continues joint training; TIES and DARE write back every isolated branch; Gated Math writes back only the admitted branch, with or without shrinkage. Without the gate, every write-back falls below the merge, whether updates are trained jointly or in isolation, dense or sparsified, because the Code and IF updates damage shared weights. With the gate but $\gamma=1$, the Math branch alone still costs $0.72$ Macro; intermediate $\gamma$ trades Math against Code and IF (Appendix~\ref{app:gamma_sweep}). The unlabelled rule selects $\gamma^\star=0.125$, the only setting above the merge.

\paragraph{Where does the gain come from?} The $+0.51$ Macro of RGBM is not a Math gain: at $\gamma^\star=0.125$, Math moves only $+0.07$. It is carried by Code ($+0.54$) and IF ($+0.92$), the two passivated domains, which the small Math write-back moves through the shared weights. Protection explains a different comparison: relative to the merge, the un-gated write-backs lose $1.1\sim 1.6$ Macro, and the gate removes that loss. Larger $\gamma$ improves Math ($+1.56$ at $\gamma=0.5$) only by sacrificing Code and IF. Across all checkpoints of Appendix~\ref{app:stage2_scatter}, the ordering un-gated write-backs $<$ merge $<$ RGBM holds.

\vspace*{-0.5\baselineskip}
\subsection{Linearity Diagnostics and the Limits of the Theory}
\label{sec:exp_linearity}
\label{sec:analysis_linearity_code}
\vspace*{-0.5\baselineskip}

Appendix Figure~\ref{fig:diagnostics} is a sanity check of Proposition~\ref{prop:homomorphism} on the isotropic family $\theta(\alpha)=W_{\mathrm{ref}}+\alpha\sum_d\tau_d$ of Remark~\ref{rem:task_arith}, not on the Minimax vector $\boldsymbol{\alpha}^\star$. Regressing the merged log-ratio on the superposed reward over 38M tokens yields $R^2=0.935$ at $\alpha=0.75$. At $\alpha=1.50$ the fit breaks ($R^2=0.557$), the cross-domain interference ratio $\mathcal{I}_{\mathrm{cross}}$ (the share of merged log-ratio energy not explained by the in-domain reward) rises to $0.775$, and Macro falls to $22.48$---a failure $\mathrm{Err}_{\mathrm{lin}}$ flags without labels. Branch updates are not small relative to the merge ($\|\Delta_d\|/\|\mathcal{T}\boldsymbol{\alpha}^\star\|\approx 22\%\sim 29\%$), and the unshrunk Math write-back already breaks the fit ($\mathrm{Err}_{\mathrm{lin}}=0.72$ at $\gamma=1$); shrinking to $\gamma^\star=0.125$ brings it back to $0.070\le\delta_{\mathrm{tol}}$ (Appendix Table~\ref{tab:p1_ablation_app}).

\vspace*{-1\baselineskip}
\section{Related Work}
\label{sec:related_work}
\vspace*{-1\baselineskip}

\paragraph{Parameter-space merging.} Training-free fusion spans Task Arithmetic~\citep{ilharco2022patching,ilharco2023editing}, sparsification~\citep{yadav2023ties,yu2024dare,jiang2025breadcrumbs}, routing~\citep{lu2025twin}, and second-order schemes such as Fisher averaging~\citep{matena2022merging,he2025localfisher} and RegMean~\citep{jin2023dataless}. Whereas AdaMerging~\citep{yang2024adamerging} optimizes black-box weights, Phase~1 obtains merge coefficients from a residual program on unlabelled prompts and normalizes the resulting per-domain residuals for diagnosis; Remark~\ref{rem:task_arith} identifies the related subspace Fisher projection of the quadratic. Extending tangent-space analyses of merging~\citep{ortizjimenez2023tangent}, Proposition~\ref{prop:homomorphism} establishes reward-space linearity, and $\mathrm{Err}_{\mathrm{lin}}$ diagnoses departures from it.

\vspace*{-0.5\baselineskip}
\paragraph{On-policy and multi-teacher distillation.} Beyond classical distillation~\citep{hinton2015distilling,kim2016sequence,you2017learning}, on-policy distillation trains on student rollouts~\citep{gu2023knowledge,agarwal2024onpolicy,gudibande2024false,zhang2025steppolicy,chen2025rdistill,liu2026prefix,yan2026jieping,liu2025did}; multi-teacher variants~\citep{openmopd2026,d3mopd2026,gopd2026,li2025policyfiltration} consolidate specialists via joint training from base weights. We show that, to second order, this objective has an optimal closed-form stage, and we gate subsequent updates by the normalized residual spectrum.

\vspace*{-1\baselineskip}
\section{Conclusion}
\label{sec:conclusion}
\vspace*{-1\baselineskip}

We bridge parameter merging and multi-teacher distillation by showing that to second order, the MOPD objective coincides with Fisher merging. Parameter merging is thus optimal, to second order, as a closed-form stage of distillation. For tractability, we restrict that merge to the task-vector subspace and, because the domain weights are unknown, replace them by their worst case: a residual program on unlabelled prompts yields the task-vector weights, and post-hoc reward-energy normalization yields the domain residual spectrum. Empirically, our zero-training merge beats a distilled student on Open-MOPD and approaches the strongest trained baseline on G-OPD, while residual-gated RGBM advances performance further. Un-gated write-backs degrade below the merge, confirming that post-merge gains are localized, with the normalized spectrum identifying the updates to test for safe write-back. Limitations are discussed in Appendix~\ref{app:limitations}.

\section{LLM Usage}
\label{sec:llm_usage}

In accordance with conference guidelines, we disclose the use of large language models (LLMs) in the preparation of this manuscript. Mainstream LLMs were used exclusively for grammar correction, stylistic polishing, and editorial refinement to enhance the readability and fluency of the text.

\clearpage
\bibliography{iclr2027_conference,my_paper}

@inproceedings{ilharco2023editing,
  author       = {Gabriel Ilharco and
                  Marco T{\'{u}}lio Ribeiro and
                  Mitchell Wortsman and
                  Ludwig Schmidt and
                  Hannaneh Hajishirzi and
                  Ali Farhadi},
  title        = {Editing models with task arithmetic},
  booktitle    = {The Eleventh International Conference on Learning Representations,
                  {ICLR} 2023, Kigali, Rwanda, May 1-5, 2023},
  publisher    = {OpenReview.net},
  year         = {2023},
  url          = {https://openreview.net/forum?id=6t0Kwf8-jrj},
  bibsource    = {dblp computer science bibliography, https://dblp.org}
}

@inproceedings{yadav2023ties,
  author       = {Prateek Yadav and
                  Derek Tam and
                  Leshem Choshen and
                  Colin A. Raffel and
                  Mohit Bansal},
  editor       = {Alice Oh and
                  Tristan Naumann and
                  Amir Globerson and
                  Kate Saenko and
                  Moritz Hardt and
                  Sergey Levine},
  title        = {TIES-Merging: Resolving Interference When Merging Models},
  booktitle    = {Advances in Neural Information Processing Systems 36: Annual Conference
                  on Neural Information Processing Systems 2023, NeurIPS 2023, New Orleans,
                  LA, USA, December 10 - 16, 2023},
  year         = {2023},
  url          = {http://papers.nips.cc/paper\_files/paper/2023/hash/1644c9af28ab7916874f6fd6228a9bcf-Abstract-Conference.html},
  bibsource    = {dblp computer science bibliography, https://dblp.org}
}

@inproceedings{yu2024dare,
  author       = {Le Yu and
                  Bowen Yu and
                  Haiyang Yu and
                  Fei Huang and
                  Yongbin Li},
  editor       = {Ruslan Salakhutdinov and
                  Zico Kolter and
                  Katherine A. Heller and
                  Adrian Weller and
                  Nuria Oliver and
                  Jonathan Scarlett and
                  Felix Berkenkamp},
  title        = {Language Models are Super Mario: Absorbing Abilities from Homologous
                  Models as a Free Lunch},
  booktitle    = {Forty-first International Conference on Machine Learning, {ICML} 2024,
                  Vienna, Austria, July 21-27, 2024},
  series       = {Proceedings of Machine Learning Research},
  volume       = {235},
  pages        = {57755--57775},
  publisher    = {{PMLR} / OpenReview.net},
  year         = {2024},
  url          = {https://proceedings.mlr.press/v235/yu24p.html},
  bibsource    = {dblp computer science bibliography, https://dblp.org}
}

@inproceedings{zheng2024expo,
  author       = {Chujie Zheng and
                  Ziqi Wang and
                  Heng Ji and
                  Minlie Huang and
                  Nanyun Peng},
  title        = {Weak-to-Strong Extrapolation Expedites Alignment},
  journal      = {CoRR},
  volume       = {abs/2404.16792},
  year         = {2024},
  url          = {https://doi.org/10.48550/arXiv.2404.16792},
  doi          = {10.48550/ARXIV.2404.16792},
  eprinttype   = {arXiv},
  eprint       = {2404.16792},
  bibsource    = {dblp computer science bibliography, https://dblp.org}
}

@article{gopd2026,
  author       = {Wenkai Yang and
                  Weijie Liu and
                  Ruobing Xie and
                  Kai Yang and
                  Saiyong Yang and
                  Yankai Lin},
  title        = {Learning beyond Teacher: Generalized On-Policy Distillation with Reward
                  Extrapolation},
  journal      = {CoRR},
  volume       = {abs/2602.12125},
  year         = {2026},
  url          = {https://doi.org/10.48550/arXiv.2602.12125},
  doi          = {10.48550/ARXIV.2602.12125},
  eprinttype   = {arXiv},
  eprint       = {2602.12125},
  bibsource    = {dblp computer science bibliography, https://dblp.org}
}

@article{openmopd2026,
  author       = {Huan{-}ang Gao and
                  Haohan Chi and
                  Yong Yan and
                  Shiyuan Feng and
                  Hanlin Wu and
                  Zheng Jiang and
                  Bingxiang He and
                  Wei{-}Ying Ma and
                  Ya{-}Qin Zhang and
                  Hao Zhou},
  title        = {Open-MOPD: Diagnosing and Fixing Capability Imbalance in Multi-Teacher
                  On-Policy Distillation},
  journal      = {CoRR},
  volume       = {abs/2608.19098},
  year         = {2026},
  url          = {https://doi.org/10.48550/arXiv.2608.19098},
  doi          = {10.48550/ARXIV.2608.19098},
  eprinttype   = {arXiv},
  eprint       = {2608.19098},
  bibsource    = {dblp computer science bibliography, https://dblp.org}
}

@article{d3mopd2026,
  author       = {Zechen Sun and
                  Zhiwei Zhang and
                  Fei Zhao and
                  Juntao Li and
                  Mu Chuan and
                  Huayu Deng and
                  Guojian Zhan and
                  Wenliang Chen and
                  Yao Hu and
                  Min Zhang},
  title        = {D\({}^{\mbox{3}}\)-MOPD: Adaptive Dynamic Domain ScheDuling for Efficient
                  Multi-Teacher Distillation},
  journal      = {CoRR},
  volume       = {abs/2608.24987},
  year         = {2026},
  url          = {https://doi.org/10.48550/arXiv.2608.24987},
  doi          = {10.48550/ARXIV.2608.24987},
  eprinttype   = {arXiv},
  eprint       = {2608.24987},
  bibsource    = {dblp computer science bibliography, https://dblp.org}
}

@inproceedings{rafailov2023direct,
  author       = {Rafael Rafailov and
                  Archit Sharma and
                  Eric Mitchell and
                  Christopher D. Manning and
                  Stefano Ermon and
                  Chelsea Finn},
  editor       = {Alice Oh and
                  Tristan Naumann and
                  Amir Globerson and
                  Kate Saenko and
                  Moritz Hardt and
                  Sergey Levine},
  title        = {Direct Preference Optimization: Your Language Model is Secretly a
                  Reward Model},
  booktitle    = {Advances in Neural Information Processing Systems 36: Annual Conference
                  on Neural Information Processing Systems 2023, NeurIPS 2023, New Orleans,
                  LA, USA, December 10 - 16, 2023},
  year         = {2023},
  url          = {http://papers.nips.cc/paper\_files/paper/2023/hash/a85b405ed65c6477a4fe8302b5e06ce7-Abstract-Conference.html},
  bibsource    = {dblp computer science bibliography, https://dblp.org}
}

@inproceedings{gu2023knowledge,
  author       = {Yuxian Gu and
                  Li Dong and
                  Furu Wei and
                  Minlie Huang},
  title        = {MiniLLM: Knowledge Distillation of Large Language Models},
  booktitle    = {The Twelfth International Conference on Learning Representations,
                  {ICLR} 2024, Vienna, Austria, May 7-11, 2024},
  publisher    = {OpenReview.net},
  year         = {2024},
  url          = {https://openreview.net/forum?id=5h0qf7IBZZ},
  bibsource    = {dblp computer science bibliography, https://dblp.org}
}

@inproceedings{agarwal2024onpolicy,
  author       = {Rishabh Agarwal and
                  Nino Vieillard and
                  Yongchao Zhou and
                  Piotr Stanczyk and
                  Sabela Ramos Garea and
                  Matthieu Geist and
                  Olivier Bachem},
  title        = {On-Policy Distillation of Language Models: Learning from Self-Generated
                  Mistakes},
  booktitle    = {The Twelfth International Conference on Learning Representations,
                  {ICLR} 2024, Vienna, Austria, May 7-11, 2024},
  publisher    = {OpenReview.net},
  year         = {2024},
  url          = {https://openreview.net/forum?id=3zKtaqxLhW},
  bibsource    = {dblp computer science bibliography, https://dblp.org}
}

@inproceedings{matena2022merging,
  author       = {Michael Matena and
                  Colin Raffel},
  editor       = {Sanmi Koyejo and
                  S. Mohamed and
                  A. Agarwal and
                  Danielle Belgrave and
                  K. Cho and
                  A. Oh},
  title        = {Merging Models with Fisher-Weighted Averaging},
  booktitle    = {Advances in Neural Information Processing Systems 35: Annual Conference
                  on Neural Information Processing Systems 2022, NeurIPS 2022, New Orleans,
                  LA, USA, November 28 - December 9, 2022},
  year         = {2022},
  url          = {http://papers.nips.cc/paper\_files/paper/2022/hash/70c26937fbf3d4600b69a129031b66ec-Abstract-Conference.html},
  bibsource    = {dblp computer science bibliography, https://dblp.org}
}

@inproceedings{jin2023dataless,
  author       = {Xisen Jin and
                  Xiang Ren and
                  Daniel Preotiuc{-}Pietro and
                  Pengxiang Cheng},
  title        = {Dataless Knowledge Fusion by Merging Weights of Language Models},
  booktitle    = {The Eleventh International Conference on Learning Representations,
                  {ICLR} 2023, Kigali, Rwanda, May 1-5, 2023},
  publisher    = {OpenReview.net},
  year         = {2023},
  url          = {https://openreview.net/forum?id=FCnohuR6AnM},
  bibsource    = {dblp computer science bibliography, https://dblp.org}
}

@inproceedings{ortizjimenez2023tangent,
  author       = {Guillermo Ortiz{-}Jim{\'{e}}nez and
                  Alessandro Favero and
                  Pascal Frossard},
  editor       = {Alice Oh and
                  Tristan Naumann and
                  Amir Globerson and
                  Kate Saenko and
                  Moritz Hardt and
                  Sergey Levine},
  title        = {Task Arithmetic in the Tangent Space: Improved Editing of Pre-Trained
                  Models},
  booktitle    = {Advances in Neural Information Processing Systems 36: Annual Conference
                  on Neural Information Processing Systems 2023, NeurIPS 2023, New Orleans,
                  LA, USA, December 10 - 16, 2023},
  year         = {2023},
  url          = {http://papers.nips.cc/paper\_files/paper/2023/hash/d28077e5ff52034cd35b4aa15320caea-Abstract-Conference.html},
  bibsource    = {dblp computer science bibliography, https://dblp.org}
}

@inproceedings{sheng2024verl,
  author       = {Guangming Sheng and
                  Chi Zhang and
                  Zilingfeng Ye and
                  Xibin Wu and
                  Wang Zhang and
                  Ru Zhang and
                  Yanghua Peng and
                  Haibin Lin and
                  Chuan Wu},
  title        = {HybridFlow: {A} Flexible and Efficient {RLHF} Framework},
  booktitle    = {Proceedings of the Twentieth European Conference on Computer Systems,
                  EuroSys 2025, Rotterdam, The Netherlands, 30 March 2025 - 3 April
                  2025},
  pages        = {1279--1297},
  publisher    = {{ACM}},
  year         = {2025},
  url          = {https://doi.org/10.1145/3689031.3696075},
  doi          = {10.1145/3689031.3696075},
  bibsource    = {dblp computer science bibliography, https://dblp.org}
}

@inproceedings{jain2024livecodebench,
  author       = {Naman Jain and
                  King Han and
                  Alex Gu and
                  Wen{-}Ding Li and
                  Fanjia Yan and
                  Tianjun Zhang and
                  Sida Wang and
                  Armando Solar{-}Lezama and
                  Koushik Sen and
                  Ion Stoica},
  title        = {LiveCodeBench: Holistic and Contamination Free Evaluation of Large
                  Language Models for Code},
  booktitle    = {The Thirteenth International Conference on Learning Representations,
                  {ICLR} 2025, Singapore, April 24-28, 2025},
  publisher    = {OpenReview.net},
  year         = {2025},
  url          = {https://openreview.net/forum?id=chfJJYC3iL},
  bibsource    = {dblp computer science bibliography, https://dblp.org}
}

@article{zhou2023instruction,
  author       = {Jeffrey Zhou and
                  Tianjian Lu and
                  Swaroop Mishra and
                  Siddhartha Brahma and
                  Sujoy Basu and
                  Yi Luan and
                  Denny Zhou and
                  Le Hou},
  title        = {Instruction-Following Evaluation for Large Language Models},
  journal      = {CoRR},
  volume       = {abs/2311.07911},
  year         = {2023},
  url          = {https://doi.org/10.48550/arXiv.2311.07911},
  doi          = {10.48550/ARXIV.2311.07911},
  eprinttype   = {arXiv},
  eprint       = {2311.07911},
  bibsource    = {dblp computer science bibliography, https://dblp.org}
}

@article{pyatkin2025ifbench,
  author       = {Valentina Pyatkin and
                  Saumya Malik and
                  Victoria Graf and
                  Hamish Ivison and
                  Shengyi Huang and
                  Pradeep Dasigi and
                  Nathan Lambert and
                  Hanna Hajishirzi},
  editor       = {Danielle Belgrave and
                  Cheng Zhang and
                  Laura N. Montoya and
                  Hsuan{-}Tien Lin and
                  Razvan Pascanu and
                  Piotr Koniusz and
                  Marzyeh Ghassemi and
                  Nancy Chen and
                  Iv{\'{a}}n Vladimir Meza Ru{\'{\i}}z and
                  Arturo Loaiza{-}Bonilla},
  title        = {Generalizing Verifiable Instruction Following},
  booktitle    = {Advances in Neural Information Processing Systems 38: Annual Conference
                  on Neural Information Processing Systems 2025, NeurIPS 2025, San Diego,
                  CA, USA, December 2-7, 2025 / Mexico City, Mexico, November 30 - December
                  5, 2025},
  year         = {2025},
  url          = {http://papers.nips.cc/paper\_files/paper/2025/hash/46499a0622ecf568b72d17b61e45dbd5-Abstract-Datasets\_and\_Benchmarks\_Track.html},
  bibsource    = {dblp computer science bibliography, https://dblp.org}
}

@article{qwen3_2025,
  author       = {Qwen Team},
  title        = {Qwen3 Technical Report},
  journal      = {CoRR},
  volume       = {abs/2505.09388},
  year         = {2025},
  url          = {https://doi.org/10.48550/arXiv.2505.09388},
  doi          = {10.48550/ARXIV.2505.09388},
  eprinttype   = {arXiv},
  eprint       = {2505.09388},
  bibsource    = {dblp computer science bibliography, https://dblp.org}
}

@misc{smollm3_2025,
  title={{SmolLM3: smol, multilingual, long-context reasoner}},
  author={Bakouch, Elie and Ben Allal, Loubna and Lozhkov, Anton and Tazi, Nouamane and Tunstall, Lewis and Patiño, Carlos Miguel and Beeching, Edward and Roucher, Aymeric and Reedi, Aksel Joonas and Gallouédec, Quentin and Rasul, Kashif and Habib, Nathan and Fourrier, Clémentine and Kydlicek, Hynek and Penedo, Guilherme and Larcher, Hugo and Morlon, Mathieu and Srivastav, Vaibhav and Lochner, Joshua and Nguyen, Xuan-Son and Raffel, Colin and von Werra, Leandro and Wolf, Thomas},
  year={2025},
  howpublished={\url{https://huggingface.co/blog/smollm3}}
}

@article{shao2024deepseekmath,
  author       = {Zhihong Shao and
                  Peiyi Wang and
                  Qihao Zhu and
                  Runxin Xu and
                  Junxiao Song and
                  Mingchuan Zhang and
                  Y. K. Li and
                  Y. Wu and
                  Daya Guo},
  title        = {DeepSeekMath: Pushing the Limits of Mathematical Reasoning in Open
                  Language Models},
  journal      = {CoRR},
  volume       = {abs/2402.03300},
  year         = {2024},
  url          = {https://doi.org/10.48550/arXiv.2402.03300},
  doi          = {10.48550/ARXIV.2402.03300},
  eprinttype   = {arXiv},
  eprint       = {2402.03300},
  bibsource    = {dblp computer science bibliography, https://dblp.org}
}

@article{tang2024merging,
  author       = {Enneng Yang and
                  Li Shen and
                  Guibing Guo and
                  Xing{-}Wei Wang and
                  Xiaochun Cao and
                  Jie Zhang and
                  Dacheng Tao},
  title        = {Model Merging in LLMs, MLLMs, and Beyond: Methods, Theories, Applications,
                  and Opportunities},
  journal      = {{ACM} Comput. Surv.},
  volume       = {58},
  number       = {8},
  pages        = {216:1--216:41},
  year         = {2026},
  url          = {https://doi.org/10.1145/3787849},
  doi          = {10.1145/3787849},
  bibsource    = {dblp computer science bibliography, https://dblp.org}
}

@inproceedings{yang2024adamerging,
  author       = {Enneng Yang and
                  Zhenyi Wang and
                  Li Shen and
                  Shiwei Liu and
                  Guibing Guo and
                  Xingwei Wang and
                  Dacheng Tao},
  title        = {AdaMerging: Adaptive Model Merging for Multi-Task Learning},
  booktitle    = {The Twelfth International Conference on Learning Representations,
                  {ICLR} 2024, Vienna, Austria, May 7-11, 2024},
  publisher    = {OpenReview.net},
  year         = {2024},
  url          = {https://openreview.net/forum?id=nZP6NgD3QY},
  bibsource    = {dblp computer science bibliography, https://dblp.org}
}

@inproceedings{ilharco2022patching,
  author       = {Gabriel Ilharco and
                  Mitchell Wortsman and
                  Samir Yitzhak Gadre and
                  Shuran Song and
                  Hannaneh Hajishirzi and
                  Simon Kornblith and
                  Ali Farhadi and
                  Ludwig Schmidt},
  editor       = {Sanmi Koyejo and
                  S. Mohamed and
                  A. Agarwal and
                  Danielle Belgrave and
                  K. Cho and
                  A. Oh},
  title        = {Patching open-vocabulary models by interpolating weights},
  booktitle    = {Advances in Neural Information Processing Systems 35: Annual Conference
                  on Neural Information Processing Systems 2022, NeurIPS 2022, New Orleans,
                  LA, USA, November 28 - December 9, 2022},
  year         = {2022},
  url          = {http://papers.nips.cc/paper\_files/paper/2022/hash/bc6cddcd5d325e1c0f826066c1ad0215-Abstract-Conference.html},
  bibsource    = {dblp computer science bibliography, https://dblp.org}
}

@article{hinton2015distilling,
  author       = {Geoffrey E. Hinton and
                  Oriol Vinyals and
                  Jeffrey Dean},
  title        = {Distilling the Knowledge in a Neural Network},
  journal      = {CoRR},
  volume       = {abs/1503.02531},
  year         = {2015},
  url          = {http://arxiv.org/abs/1503.02531},
  eprinttype   = {arXiv},
  eprint       = {1503.02531},
  bibsource    = {dblp computer science bibliography, https://dblp.org}
}

@inproceedings{kim2016sequence,
  author       = {Yoon Kim and
                  Alexander M. Rush},
  editor       = {Jian Su and
                  Xavier Carreras and
                  Kevin Duh},
  title        = {Sequence-Level Knowledge Distillation},
  booktitle    = {Proceedings of the 2016 Conference on Empirical Methods in Natural
                  Language Processing, {EMNLP} 2016, Austin, Texas, USA, November 1-4,
                  2016},
  pages        = {1317--1327},
  publisher    = {The Association for Computational Linguistics},
  year         = {2016},
  url          = {https://doi.org/10.18653/v1/d16-1139},
  doi          = {10.18653/V1/D16-1139},
  bibsource    = {dblp computer science bibliography, https://dblp.org}
}

@inproceedings{you2017learning,
  author       = {Shan You and
                  Chang Xu and
                  Chao Xu and
                  Dacheng Tao},
  title        = {Learning from Multiple Teacher Networks},
  booktitle    = {Proceedings of the 23rd {ACM} {SIGKDD} International Conference on
                  Knowledge Discovery and Data Mining, Halifax, NS, Canada, August 13
                  - 17, 2017},
  pages        = {1285--1294},
  publisher    = {{ACM}},
  year         = {2017},
  url          = {https://doi.org/10.1145/3097983.3098135},
  doi          = {10.1145/3097983.3098135},
  bibsource    = {dblp computer science bibliography, https://dblp.org}
}

@inproceedings{gudibande2024false,
  author       = {Arnav Gudibande and
                  Eric Wallace and
                  Charlie Snell and
                  Xinyang Geng and
                  Hao Liu and
                  Pieter Abbeel and
                  Sergey Levine and
                  Dawn Song},
  title        = {The False Promise of Imitating Proprietary LLMs},
  journal      = {CoRR},
  volume       = {abs/2305.15717},
  year         = {2023},
  url          = {https://doi.org/10.48550/arXiv.2305.15717},
  doi          = {10.48550/ARXIV.2305.15717},
  eprinttype   = {arXiv},
  eprint       = {2305.15717},
  bibsource    = {dblp computer science bibliography, https://dblp.org}
}

@inproceedings{lu2025twin,
  author       = {Zhenyi Lu and
                  Chenghao Fan and
                  Wei Wei and
                  Xiaoye Qu and
                  Dangyang Chen and
                  Yu Cheng},
  editor       = {Amir Globersons and
                  Lester Mackey and
                  Danielle Belgrave and
                  Angela Fan and
                  Ulrich Paquet and
                  Jakub M. Tomczak and
                  Cheng Zhang},
  title        = {Twin-Merging: Dynamic Integration of Modular Expertise in Model Merging},
  booktitle    = {Advances in Neural Information Processing Systems 37: Annual Conference
                  on Neural Information Processing Systems 2024, NeurIPS 2024, Vancouver,
                  BC, Canada, December 10 - 15, 2024},
  year         = {2024},
  url          = {http://papers.nips.cc/paper\_files/paper/2024/hash/8fcd17eb91bae20d9826786d7d6be799-Abstract-Conference.html},
  bibsource    = {dblp computer science bibliography, https://dblp.org}
}

@inproceedings{jiang2025breadcrumbs,
  author       = {MohammadReza Davari and
                  Eugene Belilovsky},
  editor       = {Ales Leonardis and
                  Elisa Ricci and
                  Stefan Roth and
                  Olga Russakovsky and
                  Torsten Sattler and
                  G{\"{u}}l Varol},
  title        = {Model Breadcrumbs: Scaling Multi-task Model Merging with Sparse Masks},
  booktitle    = {Computer Vision - {ECCV} 2024 - 18th European Conference, Milan, Italy,
                  September 29-October 4, 2024, Proceedings, Part {LXXV}},
  series       = {Lecture Notes in Computer Science},
  volume       = {15133},
  pages        = {270--287},
  publisher    = {Springer},
  year         = {2024},
  url          = {https://doi.org/10.1007/978-3-031-73226-3\_16},
  doi          = {10.1007/978-3-031-73226-3\_16},
  bibsource    = {dblp computer science bibliography, https://dblp.org}
}

@article{he2025localfisher,
  author       = {Aniruddha Roy and
                  Jyoti Patel and
                  Aman Chadha and
                  Vinija Jain and
                  Amitava Das},
  title        = {AlignMerge - Alignment-Preserving Large Language Model Merging via
                  Fisher-Guided Geometric Constraints},
  journal      = {CoRR},
  volume       = {abs/2512.16245},
  year         = {2025},
  url          = {https://doi.org/10.48550/arXiv.2512.16245},
  doi          = {10.48550/ARXIV.2512.16245},
  eprinttype   = {arXiv},
  eprint       = {2512.16245},
  bibsource    = {dblp computer science bibliography, https://dblp.org}
}

@article{deepseek2025r1,
  author       = {DeepSeek{-}AI},
  title        = {DeepSeek-R1: Incentivizing Reasoning Capability in LLMs via Reinforcement
                  Learning},
  journal      = {CoRR},
  volume       = {abs/2501.12948},
  year         = {2025},
  url          = {https://doi.org/10.48550/arXiv.2501.12948},
  doi          = {10.48550/ARXIV.2501.12948},
  eprinttype   = {arXiv},
  eprint       = {2501.12948},
  bibsource    = {dblp computer science bibliography, https://dblp.org}
}

@article{moonshot2025kimi,
  author       = {Kimi Team},
  title        = {Kimi k1.5: Scaling Reinforcement Learning with LLMs},
  journal      = {CoRR},
  volume       = {abs/2501.12599},
  year         = {2025},
  url          = {https://doi.org/10.48550/arXiv.2501.12599},
  doi          = {10.48550/ARXIV.2501.12599},
  eprinttype   = {arXiv},
  eprint       = {2501.12599},
  bibsource    = {dblp computer science bibliography, https://dblp.org}
}

@article{chen2025rdistill,
  author       = {Xueqi Cheng and
                  Xugui Zhou and
                  Tyler Derr and
                  Yushun Dong},
  title        = {ReAD: Reinforcement-Guided Capability Distillation for Large Language
                  Models},
  journal      = {CoRR},
  volume       = {abs/2605.11290},
  year         = {2026},
  url          = {https://doi.org/10.48550/arXiv.2605.11290},
  doi          = {10.48550/ARXIV.2605.11290},
  eprinttype   = {arXiv},
  eprint       = {2605.11290},
  bibsource    = {dblp computer science bibliography, https://dblp.org}
}

@article{zhang2025steppolicy,
  author       = {Qiyong Zhong and
                  Mao Zheng and
                  Mingyang Song and
                  Xin Lin and
                  Jie Sun and
                  Houcheng Jiang and
                  Xiang Wang and
                  Junfeng Fang},
  title        = {{SOD:} Step-wise On-policy Distillation for Small Language Model Agents},
  journal      = {CoRR},
  volume       = {abs/2605.07725},
  year         = {2026},
  url          = {https://doi.org/10.48550/arXiv.2605.07725},
  doi          = {10.48550/ARXIV.2605.07725},
  eprinttype   = {arXiv},
  eprint       = {2605.07725},
  bibsource    = {dblp computer science bibliography, https://dblp.org}
}

@article{li2025policyfiltration,
  author       = {Tianlei Chen and
                  Jiao Ou and
                  Ziyuan Liu and
                  Ruiming Tang and
                  Jian Liang and
                  Han Li},
  title        = {Counteraction-Aware Multi-Teacher On-Policy Distillation for General
                  Capability Recovery with Domain Preservation},
  journal      = {CoRR},
  volume       = {abs/2605.27115},
  year         = {2026},
  url          = {https://doi.org/10.48550/arXiv.2605.27115},
  doi          = {10.48550/ARXIV.2605.27115},
  eprinttype   = {arXiv},
  eprint       = {2605.27115},
  bibsource    = {dblp computer science bibliography, https://dblp.org}
}

@article{liu2025efficient,
  author       = {Kaiyuan Liu and
                  Chen Shen and
                  Zhanwei Zhang and
                  Junjie Liu and
                  Xiaosong Yuan and
                  Jieping Ye},
  title        = {Efficient Reasoning Through Suppression of Self-Affirmation Reflections
                  in Large Reasoning Models},
  journal      = {CoRR},
  volume       = {abs/2506.12353},
  year         = {2025},
  url          = {https://doi.org/10.48550/arXiv.2506.12353},
  doi          = {10.48550/ARXIV.2506.12353},
  eprinttype   = {arXiv},
  eprint       = {2506.12353},
  bibsource    = {dblp computer science bibliography, https://dblp.org}
}

@article{yan2026jieping,
  title={Jieping Ye. Distribution-aligned sequence distillation for superior long-cot reasoning},
  author={Yan, Shaotian and Liu, Kaiyuan and Shen, Chen and Wang, Bing and Fan, Sinan and Zhang, Jun and Wu, Yue and Wang, Zheng},
  journal={arXiv preprint arXiv:2601.09088},
  year={2026}
}

@article{liu2026prefix,
  title={Prefix teach, suffix fade: Local teachability collapse in strong-to-weak on-policy distillation},
  author={Liu, Kaiyuan and Zhuang, Ziyuan and Bai, Yang and Wang, Bing and Weng, Rongxiang and Ye, Jieping},
  journal={arXiv preprint arXiv:2605.13643},
  year={2026}
}

@article{liu2025did,
  title={Where did this sentence come from? tracing provenance in LLM reasoning distillation},
  author={Liu, Kaiyuan and Yan, Shaotian and Miao, Rui and Wang, Bing and Shen, Chen and Zhang, Jun and Ye, Jieping},
  journal={arXiv preprint arXiv:2512.20908},
  year={2025}
}

@article{bai2026multi,
  title={Multi-Objective and Mixed-Reward Reinforcement Learning via Reward-Decorrelated Policy Optimization},
  author={Bai, Yang and Liu, Kaiyuan and Zhuang, Ziyuan and Zhou, Jiahong and Weng, Rongxiang and Chen, Xin and Wang, Jingang and Cai, Xunliang},
  journal={arXiv preprint arXiv:2605.13641},
  year={2026}
}
\bibliographystyle{iclr2027_conference}

\clearpage
\appendix

\noindent This appendix is organized as follows.
Appendix~\ref{app:proofs} proves Theorem~\ref{thm:mopd_fisher} and Proposition~\ref{prop:homomorphism}, derives the residual, its scale-normalized spectrum, and Minimax convexity, and bounds the linearization error of shrinkage re-merging.
Appendix~\ref{app:diagnostics} reports the six-benchmark Phase~1 breakdown, the six-benchmark scores of the normalized-objective merges, the empirical reward correlation, the $K=2$ pairwise merges, generation lengths, and the full G-OPD suite.
Appendix~\ref{app:extra_exp} details the joint-distillation recipes, the post-hoc $\gamma$ write-backs, the $\epsilon$ allocation cuts, training infrastructure, the cost of assembling $M_d$, and the Open-MOPD RGBM Phase~2 specification.
Appendix~\ref{app:limitations} states the scope of the linear-basin guarantee, the specialist count, and the source of the Open-MOPD margin.

\section{Mathematical Proofs and Derivations}
\label{app:proofs}

This section records the proofs deferred from Section~\ref{sec:theory}: the second-order equivalence of multi-teacher distillation to Fisher merging (Appendix~\ref{app:proof_thm1}), the first-order reward homomorphism (Appendix~\ref{app:proof_prop2}), the second-moment residual, its scale-normalized spectrum, and uniqueness of the Minimax QP (Appendix~\ref{app:proof_residuals}), and the shrinkage bound on the linearization error of a re-merged branch (Appendix~\ref{app:proof_remerge_basin}).

\subsection{Proof of Theorem~\ref{thm:mopd_fisher} (Second-Order Equivalence to Fisher Merging)}
\label{app:proof_thm1}

Let $\mathcal{X}_d$ denote the prompt space for domain $d$, and $\mathcal{Y}$ denote the token output space. For a given prompt $x \sim \mathcal{D}_d$, the reverse Kullback-Leibler (KL) divergence between student policy $\pi_\theta(\cdot \mid x)$ and teacher policy $\pi_{W_{\phi_d}}(\cdot \mid x)$ is:
\begin{equation}
D_{\mathrm{KL}}(\pi_\theta(\cdot \mid x) \,\|\, \pi_{W_{\phi_d}}(\cdot \mid x)) = \sum_{y \in \mathcal{Y}} \pi_\theta(y \mid x) \log \frac{\pi_\theta(y \mid x)}{\pi_{W_{\phi_d}}(y \mid x)}.
\end{equation}
We define the expected domain divergence function $f_d: \mathbb{R}^P \to \mathbb{R}$ as:
\begin{equation}
f_d(\theta) := \mathbb{E}_{x \sim \mathcal{D}_d} \left[ D_{\mathrm{KL}}(\pi_\theta(\cdot \mid x) \,\|\, \pi_{W_{\phi_d}}(\cdot \mid x)) \right].
\end{equation}
We perform a multivariate Taylor expansion of $f_d(\theta)$ around the expert weight point $\theta = W_{\phi_d}$:
\begin{equation}
f_d(\theta) = f_d(W_{\phi_d}) + \nabla_\theta f_d(W_{\phi_d})^\top (\theta - W_{\phi_d}) + \frac{1}{2} (\theta - W_{\phi_d})^\top \nabla^2_\theta f_d(W_{\phi_d}) (\theta - W_{\phi_d}) + \mathcal{O}(\|\theta - W_{\phi_d}\|^3).
\label{eq:taylor_expansion_app}
\end{equation}

We evaluate each term:
\begin{enumerate}
    \item \textbf{Zeroth-order term:} At $\theta = W_{\phi_d}$, the distributions are identical:
    \begin{equation}
    f_d(W_{\phi_d}) = \mathbb{E}_{x \sim \mathcal{D}_d} \left[ \sum_{y} \pi_{W_{\phi_d}}(y \mid x) \log \frac{\pi_{W_{\phi_d}}(y \mid x)}{\pi_{W_{\phi_d}}(y \mid x)} \right] = 0.
    \end{equation}
    \item \textbf{First-order gradient term:} Taking the gradient of the reverse KL with respect to $\theta$:
    \begin{align}
    \nabla_\theta D_{\mathrm{KL}}(\pi_\theta(\cdot \mid x) \,\|\, \pi_{W_{\phi_d}}(\cdot \mid x)) &= \nabla_\theta \sum_y \pi_\theta(y \mid x) \log \pi_\theta(y \mid x) \nonumber \\
    &\quad - \nabla_\theta \sum_y \pi_\theta(y \mid x) \log \pi_{W_{\phi_d}}(y \mid x) \nonumber \\
    &= \sum_y \nabla_\theta \pi_\theta(y \mid x) \left( \log \pi_\theta(y \mid x) + 1 - \log \pi_{W_{\phi_d}}(y \mid x) \right) \nonumber \\
    &= \sum_y \nabla_\theta \pi_\theta(y \mid x) \log \frac{\pi_\theta(y \mid x)}{\pi_{W_{\phi_d}}(y \mid x)} + \sum_y \nabla_\theta \pi_\theta(y \mid x).
    \end{align}
    Evaluating at $\theta = W_{\phi_d}$: the log ratio $\log(\pi_{W_{\phi_d}} / \pi_{W_{\phi_d}}) = \log 1 = 0$. Furthermore, because $\sum_y \pi_\theta(y \mid x) = 1$ for all $\theta$, the sum of gradients vanishes: $\sum_y \nabla_\theta \pi_\theta(y \mid x) = \nabla_\theta \sum_y \pi_\theta(y \mid x) = \nabla_\theta 1 = \mathbf{0}$. Therefore:
    \begin{equation}
    \nabla_\theta f_d(W_{\phi_d}) = \mathbf{0}.
    \end{equation}
    \item \textbf{Second-order Hessian term:} Differentiating the gradient expression again with respect to $\theta$:
    \begin{align}
    \nabla^2_\theta D_{\mathrm{KL}}(\pi_\theta(\cdot \mid x) \,\|\, \pi_{W_{\phi_d}}(\cdot \mid x)) &= \sum_y \nabla^2_\theta \pi_\theta(y \mid x) \left( \log \frac{\pi_\theta(y \mid x)}{\pi_{W_{\phi_d}}(y \mid x)} + 1 \right) \nonumber \\
    &\quad + \sum_y \nabla_\theta \pi_\theta(y \mid x) \nabla_\theta \log \pi_\theta(y \mid x)^\top.
    \end{align}
    Evaluating at $\theta = W_{\phi_d}$: the log ratio vanishes, leaving $\sum_y \nabla^2_\theta \pi_{W_{\phi_d}}(y \mid x) = \nabla^2_\theta (\sum_y \pi_{W_{\phi_d}}(y \mid x)) = \mathbf{0}$. Using the identity $\nabla_\theta \pi_\theta = \pi_\theta \nabla_\theta \log \pi_\theta$, the remaining term becomes:
    \begin{align}
    \nabla^2_\theta f_d(W_{\phi_d}) &= \mathbb{E}_{x \sim \mathcal{D}_d} \left[ \sum_{y} \pi_{W_{\phi_d}}(y \mid x) \nabla_\theta \log \pi_{W_{\phi_d}}(y \mid x) \nabla_\theta \log \pi_{W_{\phi_d}}(y \mid x)^\top \right] \nonumber \\
    &= \mathbb{E}_{x \sim \mathcal{D}_d, y \sim \pi_{W_{\phi_d}}(\cdot \mid x)} \left[ \nabla_\theta \log \pi_{W_{\phi_d}}(y \mid x) \nabla_\theta \log \pi_{W_{\phi_d}}(y \mid x)^\top \right] =: F_d,
    \end{align}
    which is precisely the empirical Fisher Information Matrix $F_d \in \mathbb{R}^{P \times P}$ of expert policy $\pi_{W_{\phi_d}}$.
\end{enumerate}

Substituting these three results back into the multi-teacher objective $J_{\mathrm{MOPD}}(\theta) = \sum_{d=1}^K p_d f_d(\theta)$ yields:
\begin{equation}
J_{\mathrm{MOPD}}(\theta) = \frac{1}{2} \sum_{d=1}^K p_d (\theta - W_{\phi_d})^\top F_d (\theta - W_{\phi_d}) + \mathcal{O}\left( \max_{d} \|\theta - W_{\phi_d}\|^3 \right).
\end{equation}
Taking the gradient with respect to $\theta$ and setting it to zero:
\begin{equation}
\nabla_\theta J_{\mathrm{MOPD}}(\theta) = \sum_{d=1}^K p_d F_d (\theta - W_{\phi_d}) = \mathbf{0} \implies \left( \sum_{d=1}^K p_d F_d \right) \theta = \sum_{d=1}^K p_d F_d W_{\phi_d}.
\end{equation}
Assuming $\sum_{d=1}^K p_d F_d$ is non-singular, we obtain the global unconstrained Fisher merge:
\begin{equation}
\theta_{\mathrm{Fisher}}^\star = \left( \sum_{d=1}^K p_d F_d \right)^{-1} \sum_{d=1}^K p_d F_d W_{\phi_d}.
\end{equation}
This completes the proof of Theorem~\ref{thm:mopd_fisher}. \qed

\subsection{Proof of Proposition~\ref{prop:homomorphism} (First-Order Homomorphism)}
\label{app:proof_prop2}

Let $f(x; \theta) \in \mathbb{R}^{|V|}$ denote the logit vector at a generation step under prompt-prefix $x$. The token distribution is given by the softmax function:
\begin{equation}
\pi_\theta(y \mid x) = \frac{\exp(f(x; \theta)[y])}{\sum_{y' \in V} \exp(f(x; \theta)[y'])} = \frac{\exp(f(x; \theta)[y])}{Z(x; \theta)}.
\end{equation}
Taking the logarithm gives $\log \pi_\theta(y \mid x) = f(x; \theta)[y] - \log Z(x; \theta)$. Write $\Delta\log Z_d(x):=\log Z(x;W_{\phi_d})-\log Z(x;W_{\mathrm{ref}})$ and $\Delta\log Z_{\boldsymbol{\alpha}}(x):=\log Z(x;\theta(\boldsymbol{\alpha}))-\log Z(x;W_{\mathrm{ref}})$.

Let $\theta = W_{\mathrm{ref}} + \tau$. A first-order Taylor expansion of $f(x; \theta)[y]$ around $W_{\mathrm{ref}}$ yields:
\begin{equation}
f(x; W_{\mathrm{ref}} + \tau)[y] = f(x; W_{\mathrm{ref}})[y] + \nabla_\theta f(x; W_{\mathrm{ref}})[y]^\top \tau + \mathcal{O}(\|\tau\|^2).
\end{equation}
Applying this to specialist weights $W_{\phi_d} = W_{\mathrm{ref}} + \tau_d$, the implicit reward is:
\begin{equation}
r_d(x, y) = \log \frac{\pi_{W_{\phi_d}}(y \mid x)}{\pi_{W_{\mathrm{ref}}}(y \mid x)} = \nabla_\theta f(x; W_{\mathrm{ref}})[y]^\top \tau_d - \Delta \log Z_d(x) + \mathcal{O}(\|\tau_d\|^2).
\label{eq:rd_first_order_app}
\end{equation}
The log-partition function $\log Z(x;\theta) = \log \sum_{y' \in V} \exp(f(x;\theta)[y'])$ is a smooth function of $\theta$, so the same first-order expansion around $W_{\mathrm{ref}}$ gives
\begin{equation}
\log Z(x; W_{\mathrm{ref}} + \tau) - \log Z(x; W_{\mathrm{ref}}) = \nabla_\theta \log Z(x; W_{\mathrm{ref}})^\top \tau + \mathcal{O}(\|\tau\|^2).
\label{eq:logZ_first_order_app}
\end{equation}
Taking $\tau = \tau_d$ gives $\Delta\log Z_d(x) = \nabla_\theta \log Z(x; W_{\mathrm{ref}})^\top \tau_d + \mathcal{O}(\|\tau_d\|^2)$, and taking $\tau = \mathcal{T}\boldsymbol{\alpha} = \sum_d \alpha_d \tau_d$ gives, by linearity of the inner product,
\begin{align}
\Delta\log Z_{\boldsymbol{\alpha}}(x) &= \sum_{d=1}^K \alpha_d \, \nabla_\theta \log Z(x; W_{\mathrm{ref}})^\top \tau_d + \mathcal{O}(\|\mathcal{T}\boldsymbol{\alpha}\|^2) \nonumber \\
&= \sum_{d=1}^K \alpha_d \, \Delta\log Z_d(x) + \mathcal{O}\Big(\|\mathcal{T}\boldsymbol{\alpha}\|^2 + \sum_{d} |\alpha_d| \|\tau_d\|^2\Big).
\label{eq:logZ_superpose_app}
\end{align}
Now, for any linear combination $\theta(\boldsymbol{\alpha}) = W_{\mathrm{ref}} + \sum_{d=1}^K \alpha_d \tau_d$, by linearity of the gradient operator and Eq.~\ref{eq:logZ_superpose_app}:
\begin{align}
\log \frac{\pi_{\theta(\boldsymbol{\alpha})}(y \mid x)}{\pi_{W_{\mathrm{ref}}}(y \mid x)} &= \nabla_\theta f(x; W_{\mathrm{ref}})[y]^\top \left( \sum_{d=1}^K \alpha_d \tau_d \right) - \Delta \log Z_{\boldsymbol{\alpha}}(x) + \mathcal{O}\left( \left\| \sum_d \alpha_d \tau_d \right\|^2 \right) \nonumber \\
&= \sum_{d=1}^K \alpha_d \Big( \nabla_\theta f(x; W_{\mathrm{ref}})[y]^\top \tau_d - \Delta \log Z_d(x) \Big) + \mathcal{O}\Big(\|\mathcal{T}\boldsymbol{\alpha}\|^2 + \sum_{d} |\alpha_d| \|\tau_d\|^2\Big) \nonumber \\
&= \sum_{d=1}^K \alpha_d r_d(x, y) + \mathcal{O}\Big(\|\mathcal{T}\boldsymbol{\alpha}\|^2 + \sum_{d} |\alpha_d| \|\tau_d\|^2\Big) \approx \boldsymbol{\alpha}^\top \mathbf{r}(x, y),
\end{align}
where the last equality substitutes Eq.~\ref{eq:rd_first_order_app} for each $r_d$.
This proves the First-Order Homomorphism between weight combinations and reward superposition. \qed

\subsection{Derivation of Second-Moment Residuals and Minimax Convexity}
\label{app:proof_residuals}

\paragraph{Connection to the Projected Fisher Information Metric.}
Under the first-order logit expansion (Proposition~\ref{prop:homomorphism}), for any specialist $d$ and prompt-prefix $x$:
\begin{equation}
\mathcal{T}^\top \nabla_\theta \log \pi_{W_{\phi_d}}(y \mid x) = \mathbf{r}(x, y) - \boldsymbol{\mu}_d(x) + \mathcal{O}(\max_k \|\tau_k\|^2),
\end{equation}
where $\boldsymbol{\mu}_d(x) := \mathbb{E}_{y' \sim \pi_{W_{\phi_d}}(\cdot \mid x)}[\mathbf{r}(x, y')]$. The projected Fisher Information Matrix $S_d := \mathcal{T}^\top F_d \mathcal{T} \in \mathbb{R}^{K \times K}$ satisfies:
\begin{align}
S_d &= \mathbb{E}_{x \sim \mathcal{D}_d, y \sim \pi_{W_{\phi_d}}} \left[ \big(\mathcal{T}^\top \nabla_\theta \log \pi_{W_{\phi_d}}(y \mid x)\big) \big(\mathcal{T}^\top \nabla_\theta \log \pi_{W_{\phi_d}}(y \mid x)\big)^\top \right] \nonumber \\
&\approx \mathbb{E}_{x \sim \mathcal{D}_d, y \sim \pi_{W_{\phi_d}}} \left[ (\mathbf{r}(x, y) - \boldsymbol{\mu}_d(x)) (\mathbf{r}(x, y) - \boldsymbol{\mu}_d(x))^\top \right] = \mathrm{Cov}_{\pi_{W_{\phi_d}}}(\mathbf{r}).
\end{align}
Now, expanding the second-moment matrix $M_d := \mathbb{E}_{x \sim \mathcal{D}_d, y \sim \pi_{W_{\phi_d}}} [\mathbf{r}(x, y) \mathbf{r}(x, y)^\top]$:
\begin{equation}
M_d = \mathrm{Cov}_{\pi_{W_{\phi_d}}}(\mathbf{r}) + \mathbb{E}_{x \sim \mathcal{D}_d} [\boldsymbol{\mu}_d(x) \boldsymbol{\mu}_d(x)^\top] \approx S_d + \mathbb{E}_{x \sim \mathcal{D}_d} [\boldsymbol{\mu}_d(x) \boldsymbol{\mu}_d(x)^\top].
\end{equation}
For any weight combination vector $\boldsymbol{\alpha} \in \mathbb{R}^K$, writing $\mathbf{e}_d\in\mathbb{R}^K$ for the $d$-th standard basis vector, the residual optimized in Phase~1 is
\begin{equation}
\mathcal{L}_d(\boldsymbol{\alpha})
:= (\boldsymbol{\alpha} - \mathbf{e}_d)^\top M_d (\boldsymbol{\alpha} - \mathbf{e}_d)
\approx (\boldsymbol{\alpha} - \mathbf{e}_d)^\top S_d (\boldsymbol{\alpha} - \mathbf{e}_d)
+ \mathbb{E}_{x \sim \mathcal{D}_d} \big[ \big( (\boldsymbol{\alpha} - \mathbf{e}_d)^\top \boldsymbol{\mu}_d(x) \big)^2 \big].
\end{equation}
$\mathcal{L}_d$ upper-bounds the projected Fisher term $2 J_{\mathrm{sub}, d}(\boldsymbol{\alpha}) = (\boldsymbol{\alpha} - \mathbf{e}_d)^\top S_d (\boldsymbol{\alpha} - \mathbf{e}_d)$. After solving the Minimax problem, we form the normalized residual $\overline{\mathcal{L}}_d(\boldsymbol{\alpha}) := \mathcal{L}_d(\boldsymbol{\alpha})/\mathbb{E}[r_d^2]$. This post-hoc normalization expresses the remainder as a fraction of domain-$d$ reward energy and makes the spectrum comparable across domains.

\paragraph{Term-wise Decomposition.} Writing out the quadratic with $(M_d)_{jk} = \mathbb{E}[r_j r_k]$ and then dividing by $\mathbb{E}[r_d^2]$ yields the normalized diagnostic
\begin{equation}
\overline{\mathcal{L}}_d(\boldsymbol{\alpha}) = (\alpha_d - 1)^2 + \sum_{j \neq d} \alpha_j^2 \, \frac{\mathbb{E}[r_j^2]}{\mathbb{E}[r_d^2]} + 2 (\alpha_d - 1) \sum_{j \neq d} \alpha_j \, \frac{\mathbb{E}[r_d r_j]}{\mathbb{E}[r_d^2]} + 2 \sum_{\substack{j < k \\ j, k \neq d}} \alpha_j \alpha_k \, \frac{\mathbb{E}[r_j r_k]}{\mathbb{E}[r_d^2]}.
\label{eq:residual_tradeoff}
\end{equation}
The first term is squared relative under-recovery of the in-domain reward (it vanishes at $\alpha_d = 1$). The remaining terms are cross-domain interference measured in the same units. The correlation structure of Appendix~\ref{app:reward_corr} governs the cross-terms.

\begin{proposition}[Convexity of the Minimax QP]
The objective function $\Phi(\boldsymbol{\alpha}) := \max_{d \in \{1, \dots, K\}} \mathcal{L}_d(\boldsymbol{\alpha})$ is strictly convex in $\boldsymbol{\alpha}$ whenever each $M_d \succ 0$.
\end{proposition}
\begin{proof}
Each component has Hessian $\nabla^2 \mathcal{L}_d(\boldsymbol{\alpha}) = 2 M_d$. Since the finite collection $\{M_d\}_{d=1}^K$ is positive definite, all components share a positive strong-convexity lower bound. Their pointwise maximum is therefore strongly convex, so the Minimax problem has a unique global minimizer $\boldsymbol{\alpha}^\star$. In practice, replacing $M_d$ by $M_d+\lambda I$ supplies this condition.
\end{proof}

\subsection{Theoretical Justification of Branch Re-merging within Local Tangent Basins}
\label{app:proof_remerge_basin}

A natural question is whether the branch updates $\Delta_d = \theta_d - \theta_0^\star$ obtained in Phase~2, once written back with shrinkage, remain within the regime in which the first-order reward superposition of Proposition~\ref{prop:homomorphism} is accurate. The following statement makes precise the role of the contraction factor $\gamma$.

\begin{proposition}[Shrinkage Tightens the Linearization Error]
\label{prop:local_basin_invariance}
Fix a prompt-prefix $x$ and token $y$, and let $g(\theta) := f(x;\theta)[y]$ denote the corresponding logit. Assume $g$ is twice differentiable with $\|\nabla^2_\theta g(\theta)\| \le L_H$ on a convex neighborhood $\mathcal{B}$ containing $W_{\mathrm{ref}}$, $\theta_0^\star = W_{\mathrm{ref}} + \mathcal{T}\boldsymbol{\alpha}^\star$, and the branch endpoint $\theta_0^\star + \Delta_d$. Let $\hat{g}(\theta) := g(W_{\mathrm{ref}}) + \nabla_\theta g(W_{\mathrm{ref}})^\top (\theta - W_{\mathrm{ref}})$ be the first-order model underlying Proposition~\ref{prop:homomorphism}. Then for every $\gamma \in [0, 1]$, the re-merged point $\theta_\gamma := \theta_0^\star + \gamma \Delta_d$ satisfies
\begin{equation}
\big| g(\theta_\gamma) - \hat{g}(\theta_\gamma) \big| \;\le\; \frac{L_H}{2}\, \big\| \mathcal{T}\boldsymbol{\alpha}^\star + \gamma \Delta_d \big\|^2 \;\le\; \frac{L_H}{2}\Big( \|\mathcal{T}\boldsymbol{\alpha}^\star\| + \gamma \|\Delta_d\| \Big)^2,
\label{eq:shrinkage_bound}
\end{equation}
and the right-hand side is non-decreasing in $\gamma$. In particular, the linearization error at the re-merged model is at most that of the un-shrunk branch endpoint ($\gamma = 1$) and approaches that of $\theta_0^\star$ as $\gamma \to 0$; when $\|\Delta_d\| \ll \|\mathcal{T}\boldsymbol{\alpha}^\star\|$, the excess error over $\theta_0^\star$ is $L_H \gamma \|\Delta_d\| \|\mathcal{T}\boldsymbol{\alpha}^\star\| + \mathcal{O}(\gamma^2 \|\Delta_d\|^2)$, i.e., first order in $\gamma \|\Delta_d\|$.
\end{proposition}

\begin{proof}
Since $\theta_\gamma$ is a convex combination of $\theta_0^\star$ and $\theta_0^\star + \Delta_d$, it lies in $\mathcal{B}$. Taylor's theorem with Lagrange remainder along the segment from $W_{\mathrm{ref}}$ to $\theta_\gamma$ gives $g(\theta_\gamma) - \hat{g}(\theta_\gamma) = \frac{1}{2} (\theta_\gamma - W_{\mathrm{ref}})^\top \nabla^2_\theta g(\xi) (\theta_\gamma - W_{\mathrm{ref}})$ for some $\xi \in \mathcal{B}$, and $\theta_\gamma - W_{\mathrm{ref}} = \mathcal{T}\boldsymbol{\alpha}^\star + \gamma \Delta_d$. Bounding the Hessian by $L_H$ and applying the triangle inequality yields Eq.~\ref{eq:shrinkage_bound}. Monotonicity in $\gamma$ follows because $\|\mathcal{T}\boldsymbol{\alpha}^\star\| + \gamma\|\Delta_d\|$ is affine and increasing in $\gamma$. Expanding the square gives the excess-error expression. The same argument applies to the log-partition term $\log Z(x;\theta)$, which is a smooth function of the logit vector, so the bound transfers to the log-ratio $\log \pi_{\theta_\gamma}(y \mid x) / \pi_{W_{\mathrm{ref}}}(y \mid x)$ up to a constant factor depending only on the softmax smoothness.
\end{proof}

Proposition~\ref{prop:local_basin_invariance} does not by itself assert that $\|\Delta_d\|$ is small; it states that, whatever displacement branch training produces, shrinkage can only reduce the bound on the linearization error, and that the reduction is first order in $\gamma$. On Open-MOPD the residual gate writes back only the Math branch at the unlabelled grid value $\gamma^\star=0.125$; passivated Code/IF contribute $\Delta_d=\mathbf{0}$.

\section{Phase 1 Diagnostics and Detailed Benchmark Results}
\label{app:diagnostics}

This section supports Sections~\ref{sec:theory}--\ref{sec:experiments} with the linearity diagnostics (Appendix~\ref{app:linearity_fig}), the six-benchmark Phase~1 breakdown (Appendix~\ref{app:stage1_full_tables}), the six-benchmark scores of the normalized-objective merges (Appendix~\ref{app:s1_objective}), the empirical reward correlation (Appendix~\ref{app:reward_corr}), the $K=2$ pairwise merges (Appendix~\ref{app:k2_experiments}), generation lengths (Appendix~\ref{app:length_distribution}), and the full G-OPD suite (Appendix~\ref{app:gopd_benchmarks}).

\subsection{Linearity Diagnostics on the Isotropic Family}
\label{app:linearity_fig}

Figure~\ref{fig:diagnostics} reports the diagnostics discussed in Section~\ref{sec:exp_linearity}: the isotropic scan $\theta(\alpha)=W_{\mathrm{ref}}+\alpha\sum_d\tau_d$ of Remark~\ref{rem:task_arith}, the regression $R^2$ of the merged log-ratio on the superposed reward over 38M tokens, and the cross-domain interference ratio $\mathcal{I}_{\mathrm{cross}}:=\sum_d\mathbb{E}_{\mathcal{D}_d}[(\ell_{\alpha}-\alpha r_d)^2]/\sum_d\mathbb{E}_{\mathcal{D}_d}[\ell_{\alpha}^2]$, the share of merged log-ratio energy not explained by the in-domain reward. The branch displacements used in Section~\ref{sec:exp_linearity} are $\|\mathcal{T}\boldsymbol{\alpha}^\star\|=1.025$, $\|\Delta_{\mathrm{math}}\|=0.229$, $\|\Delta_{\mathrm{code}}\|=0.299$, $\|\Delta_{\mathrm{if}}\|=0.265$.

\begin{figure}[ht]
\centering
\begin{subfigure}[b]{0.46\textwidth}
\centering
\includegraphics[width=\linewidth]{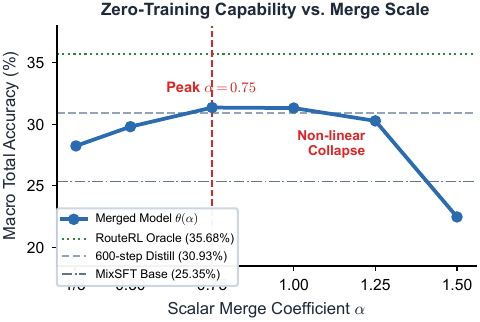}
\caption{Capability vs.\ isotropic merge scale $\alpha$}
\label{fig:diag_scan}
\end{subfigure}
\hfill
\begin{subfigure}[b]{0.46\textwidth}
\centering
\includegraphics[width=\linewidth]{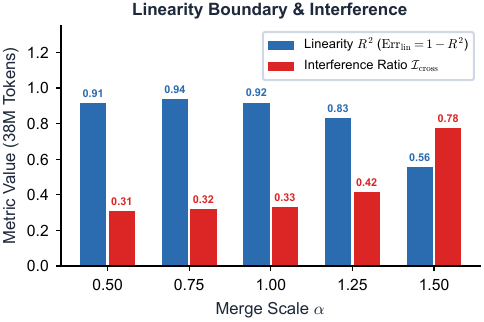}
\caption{Linearity $R^2$ and interference $\mathcal{I}_{\mathrm{cross}}$}
\label{fig:diag_linearity}
\end{subfigure}
\caption{\textbf{Linearity diagnostics on the isotropic family} ($\theta(\alpha)=W_{\mathrm{ref}}+\alpha\sum_d\tau_d$; not the Minimax vector $\boldsymbol{\alpha}^\star$). (a)~Macro vs.\ scalar $\alpha$. Horizontal lines are RouteRL ($35.68$), the $600$-step student ($30.93$), and MixSFT ($25.35$); the peak at $\alpha=0.75$ is this scan, not $\boldsymbol{\alpha}^\star$. (b)~$R^2$ and $\mathcal{I}_{\mathrm{cross}}$ over 38M tokens.}
\label{fig:diagnostics}
\end{figure}

\subsection{Complete Per-Benchmark Phase 1 Results}
\label{app:stage1_full_tables}

Table~\ref{tab:stage1_full_breakdown} details the individual benchmark evaluations across all six downstream evaluation tasks: AIME 2024, AIME 2025 (Math), LiveCodeBench v5, LiveCodeBench v6 (Code), IFEval, and IFBench (Instruction Following).
The Type column is the coefficient family: Scalar is a shared scale $\alpha\mathbf{1}$; 3D Vector is a free $\boldsymbol{\alpha}\in\mathbb{R}^3$. Official Distilled Model is the main-text Multi-Teacher OPD row. Forward-KL Merge uses forward KL in the same isotropic subspace program; Minimax Worst-Domain Merge (Scalar) is Eq.~\ref{eq:minimax_qp} restricted to $\boldsymbol{\alpha}=\alpha\mathbf{1}$. Both Subspace Fisher Merge rows solve Eq.~\ref{eq:subspace_alpha} under uniform $p_d=1/3$ with $M_d$ in place of $S_d$; the main-text Subspace Fisher Merge is the 3D Vector row.

\begin{table}[ht]
\centering
\caption{\textbf{Comprehensive benchmark evaluation of Phase~1 candidates.} Six-task breakdown. Bold: best zero-training candidate per column (or Phase~2 where it exceeds Phase~1).}
\label{tab:stage1_full_breakdown}
\small
\renewcommand{\arraystretch}{0.90}
\begin{adjustbox}{max width=\textwidth}
\begin{tabular}{lccccccccccc}
\toprule
\multirow{2}{*}{\textbf{Method / Configuration}} & \multirow{2}{*}{\textbf{Type}} & \multicolumn{3}{c}{\textbf{Math}} & \multicolumn{3}{c}{\textbf{Code}} & \multicolumn{3}{c}{\textbf{Instruction Following}} & \multirow{2}{*}{\textbf{Macro}} \\
\cmidrule(lr){3-5} \cmidrule(lr){6-8} \cmidrule(lr){9-11}
& & \textbf{A24} & \textbf{A25} & \textbf{Avg} & \textbf{v5} & \textbf{v6} & \textbf{Avg} & \textbf{IFEval} & \textbf{IFB} & \textbf{Avg} & \\
\midrule
MixSFT Base ($W_{\mathrm{ref}}$) & -- & 14.79 & 18.18 & 16.49 & 15.63 & 18.46 & 17.05 & 66.36 & 18.67 & 42.52 & 25.35 \\
RouteRL (Oracle Upper Bound) & -- & 26.88 & 27.60 & 27.24 & 27.27 & 25.26 & 26.27 & 81.70 & 25.33 & 53.52 & 35.68 \\
\midrule
Official Distilled Model & Distill & 24.38 & 24.68 & 24.53 & 21.05 & 21.31 & 21.18 & 70.24 & 24.00 & 47.12 & 30.93 \\
Equal-Weight Average & Scalar & 20.31 & 21.30 & 20.81 & 17.66 & 20.63 & 19.15 & 69.87 & 19.67 & 44.77 & 28.24 \\
Task Arithmetic & Scalar & 23.91 & 24.79 & 24.35 & 22.99 & 22.63 & 22.81 & 69.87 & 23.67 & 46.77 & 31.31 \\
\midrule
Subspace Fisher Merge & Scalar & \textbf{25.47} & 24.48 & \textbf{24.97} & 23.65 & 22.91 & 23.28 & 73.94 & 21.33 & 47.64 & 31.96 \\
Subspace Fisher Merge & 3D Vector & 23.85 & 24.38 & 24.11 & 24.49 & 22.69 & 23.59 & 72.64 & 23.33 & 47.99 & 31.90 \\
Forward KL (FKL) Merge & Scalar & 24.58 & 23.28 & 23.93 & 23.59 & 22.34 & 22.97 & 71.53 & 24.33 & 47.93 & 31.61 \\
Minimax Worst-Domain Merge & Scalar & 25.10 & 24.84 & \textbf{24.97} & 24.07 & 22.91 & 23.49 & 69.87 & 20.33 & 45.10 & 31.19 \\
\textbf{Minimax Merge ($\theta_0^\star$)} & \textbf{3D Vector} & 23.91 & 23.80 & 23.86 & 24.27 & 23.91 & 24.09 & 73.87 & \textbf{25.33} & \textbf{49.60} & \textbf{32.52} \\
\textbf{RGBM Phase 2} & Residual-Gated & 23.80 & 24.06 & 23.93 & \textbf{25.15} & \textbf{24.11} & \textbf{24.63} & \textbf{75.05} & \textbf{26.00} & \textbf{50.52} & \textbf{33.03} \\
\bottomrule
\end{tabular}
\end{adjustbox}
\end{table}

\subsection{Scores of the Normalized-Objective Merges}
\label{app:s1_objective}

The four Open-MOPD programs of Table~\ref{tab:s1_objective} that minimize $\mathcal{L}_d/\mathbb{E}[r_d^2]$ are absent from Table~\ref{tab:stage1_full_breakdown}; their six-benchmark scores are Table~\ref{tab:s1_objective_benchmarks}.

\begin{table}[ht]
\centering
\caption{\textbf{Six-benchmark scores of the four normalized-objective merges.} Same protocol as Table~\ref{tab:stage1_full_breakdown}. The four $\mathcal{L}$ programs and $\boldsymbol{\alpha}=\mathbf{1}/3$ are already in that table.}
\label{tab:s1_objective_benchmarks}
\small
\renewcommand{\arraystretch}{0.90}
\begin{adjustbox}{max width=\textwidth}
\begin{tabular}{lcccccccccc}
\toprule
\textbf{Program} & \textbf{A24} & \textbf{A25} & \textbf{Math} & \textbf{v5} & \textbf{v6} & \textbf{Code} & \textbf{IFEval} & \textbf{IFB} & \textbf{IF} & \textbf{Macro} \\
\midrule
A scalar $\overline{\mathcal{L}}$ & 24.64 & 24.11 & 24.38 & 23.23 & 23.49 & 23.36 & 71.72 & 18.00 & 44.86 & 30.86 \\
A vector $\overline{\mathcal{L}}$ & 23.65 & 24.74 & 24.19 & 17.60 & 21.54 & 19.57 & 74.68 & 21.33 & 48.00 & 30.59 \\
C scalar $\overline{\mathcal{L}}$ & 21.25 & 21.56 & 21.41 & 18.44 & 21.09 & 19.76 & 70.24 & 21.33 & 45.79 & 28.99 \\
C vector $\overline{\mathcal{L}}$ & 23.54 & 24.90 & 24.22 & 19.52 & 21.94 & 20.73 & 75.05 & 19.00 & 47.02 & 30.66 \\
\bottomrule
\end{tabular}
\end{adjustbox}
\end{table}

\subsection{Empirical Reward Correlation Geometry}
\label{app:reward_corr}

We sample 6,144 rollout sequences across domain training prompts and extract implicit rewards $\mathbf{r} = [r_{\mathrm{math}}, r_{\mathrm{code}}, r_{\mathrm{if}}]^\top$. The empirical correlation matrix is:
\begin{equation}
\mathrm{Corr}(\mathbf{r}) = \begin{bmatrix}
1.000 & \mathbf{0.313} & 0.043 \\
\mathbf{0.313} & 1.000 & 0.035 \\
0.043 & 0.035 & 1.000
\end{bmatrix}.
\end{equation}
The off-diagonal correlation between Math and Code is moderately positive ($0.313$), indicating that mathematical problem-solving and programming share underlying algorithmic reasoning representations. In contrast, the Instruction-Following specialist is strictly orthogonal to both reasoning domains ($< 0.045$). This geometric structure directly explains why the Minimax solver assigns $\alpha_{\mathrm{if}}^\star \approx 0.800$ (as the IF domain experiences near-zero destructive interference) while compressing $\alpha_{\mathrm{math}}^\star \approx 0.597$ and $\alpha_{\mathrm{code}}^\star \approx 0.620$ to minimize cross-talk.

\subsection{Two-Domain ($K=2$) Sub-Consolidations and Residual Tracking}
\label{app:k2_experiments}

To test whether the Fisher-geometric theory and the predictive validity of the residual spectrum hold across different expert subsets, we evaluate all three pairwise $K=2$ specialist consolidations formed from $\{W_{\phi_{\mathrm{math}}}, W_{\phi_{\mathrm{code}}}, W_{\phi_{\mathrm{if}}}\}$:
\begin{enumerate}
    \item \textbf{Math + Code ($K=2$):} The implicit rewards share positive covariance ($\rho = 0.313$). The Minimax solver allocates balanced weights $\boldsymbol{\alpha}^\star = [0.635, 0.652]^\top$, yielding zero-training performance of Math 24.38, Code 24.81 (Macro 24.60, higher than the 3-expert merge on Code due to zero interference from IF). The normalized spectrum ranks $\overline{\mathcal{L}}_{\mathrm{math}} > \overline{\mathcal{L}}_{\mathrm{code}}$, correctly identifying Math as the remaining active bottleneck.
    \item \textbf{Math + IF ($K=2$):} Because the two specialists are nearly orthogonal ($\rho = 0.043$), cross-domain interference is minimal. The Minimax weights are $\boldsymbol{\alpha}^\star = [0.680, 0.885]^\top$, achieving Math 24.82 and IF 48.95 (Macro 36.89). $\overline{\mathcal{L}}_{\mathrm{if}}$ is near the floor (almost complete recovery), while $\overline{\mathcal{L}}_{\mathrm{math}}$ marks Math for post-merge branch training.
    \item \textbf{Code + IF ($K=2$):} With near-zero cross-covariance ($\rho = 0.035$), linear superposition incurs minimal interference. The Minimax weights $\boldsymbol{\alpha}^\star = [0.720, 0.890]^\top$ achieve Code 25.02 and IF 49.10 (Macro 37.06). Both normalized residuals sit near the floor, predicting that both domains are passivated and require zero post-merge branch compute.
\end{enumerate}
Across all $K=2$ sub-settings, the normalized residual spectrum $\overline{\boldsymbol{\mathcal{L}}}$ correctly ranks domain bottlenecks and tracks pairwise recovery.

\subsection{Token Generation Length Distributions}
\label{app:length_distribution}

Table~\ref{tab:generation_lengths} summarizes the average completion lengths across evaluation benchmarks.

\begin{table}[ht]
\centering
\caption{\textbf{Average generation length across models.} Completion tokens on one representative task per domain (AIME 2024, LiveCodeBench v5, IFBench).}
\label{tab:generation_lengths}
\small
\renewcommand{\arraystretch}{0.90}
\begin{adjustbox}{max width=\columnwidth}
\begin{tabular}{lccc}
\toprule
\textbf{Model} & \textbf{AIME 2024} (Math) & \textbf{LiveCodeBench v5} (Code) & \textbf{IFBench} (IF) \\
\midrule
MixSFT Base ($W_{\mathrm{ref}}$) & 3,420 & 485 & 310 \\
Math Specialist ($W_{\phi_{\mathrm{math}}}$) & 14,850 & 490 & 280 \\
Code Specialist ($W_{\phi_{\mathrm{code}}}$) & 3,890 & 720 & 295 \\
IF Specialist ($W_{\phi_{\mathrm{if}}}$) & 2,950 & 420 & 315 \\
\midrule
Equal-Weight Task Arithmetic & 11,240 & 560 & 410 \\
Official Distilled Model (600 steps) & 11,850 & 580 & 390 \\
\textbf{Zero-Training Merge ($\theta_0^\star$)} & \textbf{12,316} & \textbf{615} & \textbf{461} \\
\bottomrule
\end{tabular}
\end{adjustbox}
\end{table}

The zero-training merge $\theta_0^\star$ preserves extended reasoning lengths ($\approx 12.3\text{k}$ tokens on AIME), confirming that parameter merging successfully preserves reasoning traces without requiring expensive post-hoc length penalties.

\subsection{Cross-Suite Generalization on G-OPD Benchmarks (Qwen3-4B)}
\label{app:gopd_benchmarks}

To further establish that the Fisher-geometric projection generalizes across specialist regimes and model families, we evaluate RGBM on the official G-OPD benchmarks~\citep{gopd2026} using Qwen3-4B in two specialist regimes: the intermediate checkpoints of Table~2 in \citet{gopd2026} ($W_{\phi_{\mathrm{math}}}$ at 500 steps, $W_{\phi_{\mathrm{code}}}$ at 300 steps), which are the main-text G-OPD setting, and the fully trained pair of their Table~8 (both at 1200 steps).
Table~\ref{tab:gopd_eval_results} reports the complete per-task benchmark results across all 7 evaluation datasets (AIME 2024, AIME 2025, HMMT Feb, HMMT Nov, HumanEval+ [HE+], MBPP+, and LiveCodeBench [LCB v6]) following the standardized evaluation protocol.

\begin{table}[ht]
\centering
\caption{\textbf{Per-benchmark results on the G-OPD suite (Qwen3-4B).} Intermediate and converged specialist regimes under the official 7-task protocol. ExPO is a published G-OPD baseline. Single-Teacher distills from one specialist; Multi-Teacher uses both. Bold: best consolidated model per column (teacher excluded).}
\label{tab:gopd_eval_results}
\small
\renewcommand{\arraystretch}{0.88}
\begin{adjustbox}{max width=\textwidth}
\begin{tabular}{llcccccccc}
\toprule
\textbf{Regime} & \textbf{Method / Configuration} & \textbf{AIME24} & \textbf{AIME25} & \textbf{HMMT-Feb} & \textbf{HMMT-Nov} & \textbf{HE+} & \textbf{MBPP+} & \textbf{LCB} & \textbf{7-Task Avg} \\
\midrule
\multirow{10}{*}{\shortstack[l]{\textbf{Intermediate}\\(Math-500\\+ Code-300)}}
& Teacher ($W_{\phi_{\mathrm{math}}}, W_{\phi_{\mathrm{code}}}$) & 58.00 & 54.60 & 32.50 & 38.90 & 86.00 & 70.20 & 27.30 & 52.50 \\
& Student Base ($W_{\mathrm{ref}}$) & 21.50 & 21.90 & 10.00 & 8.00 & 74.70 & 64.70 & 17.90 & 31.24 \\
\cmidrule(lr){2-10}
& SFT & 58.50 & 53.30 & 30.70 & 34.80 & 86.40 & 69.60 & 26.40 & 51.39 \\
& ExPO~\citep{zheng2024expo} & 57.50 & 54.50 & 31.70 & 36.30 & 86.70 & \textbf{72.00} & 29.00 & 52.53 \\
& On-Policy Distillation (OPD) & 60.60 & 54.10 & 32.50 & 38.30 & 84.60 & 69.50 & 27.60 & 52.46 \\
& ExOPD~\citep{gopd2026} & \textbf{61.00} & 56.00 & \textbf{34.40} & 39.20 & 86.30 & 70.60 & 29.00 & 53.79 \\
\cmidrule(lr){2-10}
& Task Arithmetic & 56.98 & 50.42 & 28.02 & 36.77 & 84.30 & 68.52 & 29.14 & 50.59 \\
& Equal-Weight Average & 33.75 & 25.00 & 15.52 & 15.94 & 79.88 & 67.13 & 20.00 & 36.75 \\
& RGBM Phase 1 (Ours) & 59.17 & 55.31 & 32.71 & 37.50 & \textbf{87.96} & 71.43 & 29.29 & 53.34 \\
& \textbf{RGBM Phase 2 (Ours)} & 60.42 & \textbf{56.56} & 33.96 & \textbf{39.27} & 86.91 & 71.30 & \textbf{30.57} & \textbf{54.14} \\
\midrule
\midrule
\multirow{9}{*}{\shortstack[l]{\textbf{Converged}\\(Math-1200\\+ Code-1200)}}
& Teacher ($W_{\phi_{\mathrm{math}}}, W_{\phi_{\mathrm{code}}}$) & 68.20 & 59.30 & 37.30 & 42.90 & 88.90 & 72.50 & 28.00 & 56.73 \\
& Student Base ($W_{\mathrm{ref}}$) & 21.50 & 21.90 & 10.00 & 8.00 & 74.70 & 64.70 & 17.90 & 31.24 \\
\cmidrule(lr){2-10}
& OPD (Single-Teacher) & 68.30 & 58.70 & \textbf{38.70} & 41.20 & 89.30 & 71.30 & 28.00 & 56.50 \\
& ExOPD (Single-Teacher) & 68.40 & 59.20 & 38.20 & 42.80 & \textbf{89.90} & 73.70 & 29.30 & 57.36 \\
& OPD (Multi-Teacher) & 68.20 & 60.20 & 38.50 & 40.80 & 86.40 & 72.10 & 27.60 & 56.26 \\
& ExOPD (Multi-Teacher) & \textbf{70.10} & 59.60 & 37.50 & 42.70 & 89.50 & 73.90 & 29.70 & \textbf{57.57} \\
\cmidrule(lr){2-10}
& Task Arithmetic & 64.79 & 59.48 & 35.10 & 41.56 & 87.50 & 74.74 & 30.86 & 56.29 \\
& Equal-Weight Average & 40.94 & 35.62 & 20.10 & 24.17 & 80.79 & 68.85 & 21.57 & 41.72 \\
& \textbf{RGBM Phase 1 (Ours)} & 67.50 & \textbf{62.92} & 36.46 & \textbf{43.33} & 85.37 & \textbf{74.87} & \textbf{31.00} & 57.35 \\
\bottomrule
\end{tabular}
\end{adjustbox}
\end{table}

\paragraph{Phase~1 coefficients and Phase~2 allocation.}
On the intermediate pair the Minimax weights are $\boldsymbol{\alpha}^\star=(0.675, 0.666)$ and the normalized spectrum $\overline{\boldsymbol{\mathcal{L}}}=(0.094, 0.090)$ has its largest gap down to $0$, so Eq.~\ref{eq:epsilon_midgap} gives $\epsilon=0.045$ and both Math and Code stay active. Isolated distillation branches re-merged at the unlabelled grid value $\gamma=0.5$ reach $54.14$ (Table~\ref{tab:gopd_eval_results}), $+0.80$ over Phase~1 and $+0.35$ over trained ExOPD. On the converged pair we report only Phase~1 ($\boldsymbol{\alpha}^\star=(0.511, 0.505)$, average $57.35$); there is no Phase~2 run.

\section{Additional Experimental Details and Parameter Sensitivity}
\label{app:extra_exp}

This section collects the full RGBM procedure and the systems compared (Appendix~\ref{app:algorithm_systems}), the five joint-distillation recipes (Appendix~\ref{app:stage2_scatter}), the post-hoc $\gamma$ write-backs (Appendix~\ref{app:gamma_sweep}), the $\epsilon$ allocation cuts (Appendix~\ref{app:epsilon_sensitivity}), training infrastructure (Appendix~\ref{app:training_params}), the cost and robustness of assembling $M_d$ (Appendix~\ref{app:md_cost_analysis}), and the Open-MOPD RGBM Phase~2 specification (Appendix~\ref{app:remerge_groups}).

\subsection{RGBM Procedure and Systems Compared}
\label{app:algorithm_systems}

Algorithm~\ref{alg:rgbm} lists the complete RGBM procedure of Section~\ref{sec:method}; Table~\ref{tab:systems} lists how every Open-MOPD system in Table~\ref{tab:master_results} is built.

\begin{algorithm}[ht]
\caption{Residual-Governed Branch-and-Merge (\textsc{RGBM})}
\label{alg:rgbm}
\footnotesize
\begin{algorithmic}[1]
\REQUIRE Base weights $W_{\mathrm{ref}}$, specialist weights $\{W_{\phi_d}\}_{d=1}^K$, unlabelled calibration prompts $\mathcal{D}_{\mathrm{calib}}=\{\mathcal{D}_{\mathrm{calib},d}\}_{d=1}^K$, Tikhonov constant $\lambda$, discrete shrinkage grid $\Gamma=\{0.125, 0.25, 0.5, 0.75\}$, linearity tolerance $\delta_{\mathrm{tol}}$.
\STATE For each $d$: draw $N_d$ pairs $x_i \sim \mathcal{D}_{\mathrm{calib}, d}$, $y_i \sim \pi_{W_{\phi_d}}(\cdot \mid x_i)$; score $\mathbf{r}(x_i, y_i) = [\log \frac{\pi_{W_{\phi_k}}(y_i \mid x_i)}{\pi_{W_{\mathrm{ref}}}(y_i \mid x_i)}]_{k=1}^K$; $M_d \leftarrow \frac{1}{N_d}\sum_i \mathbf{r}\mathbf{r}^\top$.
\STATE Solve the residual Minimax QP (Eq.~\ref{eq:minimax_qp}) with $\tilde M_d=M_d+\lambda I$ for $\boldsymbol{\alpha}^\star$; set $\theta_0^\star \leftarrow W_{\mathrm{ref}} + \mathcal{T}\boldsymbol{\alpha}^\star$; using $M_d$, set $\mathcal{L}_d \leftarrow (\boldsymbol{\alpha}^\star - \mathbf{e}_d)^\top M_d (\boldsymbol{\alpha}^\star - \mathbf{e}_d)$ and $\overline{\mathcal{L}}_d \leftarrow \mathcal{L}_d/\mathbb{E}[r_d^2]$ (Eq.~\ref{eq:normalized_residual}).
\STATE Sort $\overline{\mathcal{L}}_{(1)}\ge\cdots\ge\overline{\mathcal{L}}_{(K)}\ge\overline{\mathcal{L}}_{(K+1)}:=0$; $k^\star\leftarrow\arg\max_i(\overline{\mathcal{L}}_{(i)}-\overline{\mathcal{L}}_{(i+1)})$; $\epsilon\leftarrow\frac12(\overline{\mathcal{L}}_{(k^\star)}+\overline{\mathcal{L}}_{(k^\star+1)})$ (Eq.~\ref{eq:epsilon_midgap}).
\FOR{$d = 1, \dots, K$}
    \IF{$\overline{\mathcal{L}}_d \le \epsilon$}
        \STATE \textbf{Passivate:} $\Delta_d \leftarrow \mathbf{0}$
    \ELSE
        \STATE \textbf{Active branch:} $\theta_d \leftarrow \theta_0^\star$; train on $\mathcal{D}_d$ only; $\Delta_d \leftarrow \theta_d^{(\mathrm{end})} - \theta_0^\star$.
        \STATE Compute unlabelled shrinkage $\gamma_d^\star \leftarrow \max \{\gamma \in \Gamma \mid \mathrm{Err}_{\mathrm{lin}}(\theta_0^\star + \gamma \Delta_d) \le \delta_{\mathrm{tol}}\}$ (Eq.~\ref{eq:rgbm_remerge_eq}).
    \ENDIF
\ENDFOR
\STATE \textbf{Re-merge:} $\theta_{\mathrm{RGBM}} \leftarrow \theta_0^\star + \sum_{d: \overline{\mathcal{L}}_d > \epsilon} \gamma_d^\star \, \Delta_d$ (Eq.~\ref{eq:rgbm_remerge_eq}).
\RETURN $\theta_{\mathrm{RGBM}}$
\end{algorithmic}
\end{algorithm}

When evaluating $\mathrm{Err}_{\mathrm{lin}}(\theta_0^\star+\gamma\Delta_d)$ in Algorithm~\ref{alg:rgbm}, the response is $\ell_{\theta_\gamma}=\log(\pi_{\theta_\gamma}/\pi_{W_{\mathrm{ref}}})$ and the predictor is the Phase~1 superposition $\boldsymbol{\alpha}^{\star\top}\mathbf{r}$, held fixed. Both are scored token-wise on the same Phase~1 pairs $(x,y)$ used to assemble $M_d$, pooled over $\bigcup_d\mathcal{D}_{\mathrm{calib},d}$. A single OLS with intercept then yields $R^2$, and $\mathrm{Err}_{\mathrm{lin}}=1-R^2$ as in Eq.~\ref{eq:err_lin_def}.

\begin{table}[ht]
\centering
\caption{\textbf{Consolidation systems compared on Open-MOPD.} How each row of Table~\ref{tab:master_results} is built: initialization, training, and what is written back.}
\label{tab:systems}
\footnotesize
\renewcommand{\arraystretch}{0.85}\setlength{\tabcolsep}{4pt}
\begin{adjustbox}{max width=\textwidth}
\begin{tabular}{llll}
\toprule
\textbf{System} & \textbf{Init} & \textbf{Training} & \textbf{Written back} \\
\midrule
RouteRL~\citep{openmopd2026} & $K$ specialists & none & Routing oracle; serves $K$ models \\
Multi-Teacher OPD~\citep{openmopd2026} & $W_{\mathrm{ref}}$ & joint OPD, all domains & Full student checkpoint \\
\midrule
Task Arithmetic~\citep{ilharco2023editing} & $W_{\mathrm{ref}}$ & none & $\sum_d \tau_d$ ($\alpha=1$) \\
Subspace Fisher Merge & $W_{\mathrm{ref}}$ & none & $\mathcal{T}\boldsymbol{\alpha}^\star_{\mathrm{sub}}$ (Eq.~\ref{eq:subspace_alpha}, $p_d=1/3$, $M_d$ for $S_d$) \\
RGBM Phase 1 ($\theta_0^\star$) & $W_{\mathrm{ref}}$ & none & $\mathcal{T}\boldsymbol{\alpha}^\star$ \\
\midrule
Uniform Distillation & $\theta_0^\star$ & joint OPD, $p_d=1/3$ & Shared weights \\
TIES~\citep{yadav2023ties} / DARE~\citep{yu2024dare} & $\theta_0^\star$ & isolated branch per domain & All $\{\Delta_d\}_{d=1}^3$, sparsified \\
Gated Math ($\gamma=1.0$) & $\theta_0^\star$ & isolated Math branch only & $\Delta_{\mathrm{math}}$ (no shrinkage) \\
Gated Math ($\gamma=0.5$) & $\theta_0^\star$ & isolated Math branch only & $0.5\,\Delta_{\mathrm{math}}$ \\
RGBM Phase 2 (Ours) & $\theta_0^\star$ & the same isolated Math branch & $\gamma^\star\Delta_{\mathrm{math}}$ \\
\bottomrule
\end{tabular}
\end{adjustbox}
\end{table}

\subsection{Joint Multi-Task Distillation: Exploration of 5 Recipes}
\label{app:stage2_scatter}

To thoroughly investigate whether any joint fine-tuning objective could maintain or improve upon $\theta_0^\star$, we implemented and evaluated five distinct joint training recipes:
\begin{enumerate}
    \item \textbf{Standard Multi-Teacher Distillation (Uniform MOPD):} Multi-teacher reverse-KL on student rollouts with equal domain sampling ($p_d = 1/3$).
    \item \textbf{Outcome-Gated Distillation:} Outcome-gated on-policy distillation, where teacher distribution matching is active only when student reasoning fails verification.
    \item \textbf{Outcome-Veto Regularized Distillation:} Regularized distillation applying explicit negative KL penalties to suppress drift on non-target domains.
    \item \textbf{Trust-Region Regularized Distillation:} Outcome-gated distillation with explicit KL anchoring to $\theta_0^\star$ ($\beta = 2.5$).
    \item \textbf{Hybrid Teacher-Anchor Distillation:} Convex combination of specialist teacher KL and student-to-$\theta_0^\star$ anchor KL.
\end{enumerate}

\paragraph{Empirical Findings across Checkpoints.} On the five recipes in Table~\ref{tab:joint_trajectories_app}, Macro Total stays in the $30.5\%\sim 32.3\%$ range; every listed checkpoint lies below residual-gated RGBM Phase~2 ($33.03\%$). At the last listed checkpoint of each recipe, Code is $1.1\sim 2.9$ points below $\theta_0^\star$, and IF stays below the Phase~1 value $49.60$. Within this testbed, no joint objective we tried preserved the low-residual domains while training on the shared weights, which is the behaviour anticipated by the argument of Section~\ref{sec:method_phase2}.

\begin{table}[ht]
\centering
\caption{\textbf{Trajectory evaluation of joint multi-teacher distillation from $\theta_0^\star$.} Five joint objectives. Trajectory Mean is the average of the listed checkpoints. Trust-Region is scored at steps 5/10/20 because that run ends at step 20.}
\label{tab:joint_trajectories_app}
\small
\renewcommand{\arraystretch}{0.85}
\begin{adjustbox}{max width=\columnwidth}
\begin{tabular}{lcccccc}
\toprule
\textbf{Distillation Framework} & \textbf{Step} & \textbf{Math} & \textbf{Code} & \textbf{IF} & \textbf{Macro Total} & $\Delta \text{ vs. } \theta_0^\star$ \\
\midrule
RGBM Phase 1 ($\theta_0^\star$) & 0 & 23.86 & 24.09 & 49.60 & 32.52 & -- \\
\textbf{RGBM Phase 2 (Ours)} & 0 & 23.93 & 24.63 & 50.52 & \textbf{33.03} & \textbf{+0.51} \\
\midrule
\multirow{3}{*}{\shortstack[l]{\textbf{Uniform Distillation}\\\citep{openmopd2026}}}
& Step 10 & 24.38 & 20.78 & 47.79 & 30.98 & -1.54 \\
& Step 30 & 24.58 & 21.46 & 47.62 & 31.22 & -1.30 \\
& Step 50 & 22.94 & 21.83 & 46.84 & 30.54 & -1.98 \\
\cmidrule(lr){2-7}
& \textbf{Trajectory Mean} & 23.97 & 21.36 & 47.42 & \textbf{30.91} & \textbf{-1.61} \\
\midrule
\multirow{3}{*}{\shortstack[l]{\textbf{Outcome-Gated Distillation}}}
& Step 10 & 24.45 & 21.62 & 48.12 & 31.40 & -1.12 \\
& Step 30 & 24.97 & 22.18 & 48.91 & 32.02 & -0.50 \\
& Step 50 & 24.64 & 22.25 & 46.97 & 31.28 & -1.24 \\
\cmidrule(lr){2-7}
& \textbf{Trajectory Mean} & 24.69 & 22.02 & 48.00 & \textbf{31.57} & \textbf{-0.95} \\
\midrule
\multirow{3}{*}{\shortstack[l]{\textbf{Outcome-Veto Distillation}}}
& Step 10 & 25.23 & 21.93 & 47.10 & 31.42 & -1.10 \\
& Step 30 & 24.24 & 21.59 & 48.99 & 31.61 & -0.91 \\
& Step 50 & 25.08 & 21.20 & 47.73 & 31.33 & -1.19 \\
\cmidrule(lr){2-7}
& \textbf{Trajectory Mean} & 24.85 & 21.57 & 47.94 & \textbf{31.45} & \textbf{-1.07} \\
\midrule
\multirow{3}{*}{\shortstack[l]{\textbf{Trust-Region Regularization}\\($\beta = 2.5$)}}
& Step 5 & 23.57 & 23.33 & 49.99 & 32.30 & -0.22 \\
& Step 10 & 24.84 & 23.66 & 47.52 & 32.01 & -0.51 \\
& Step 20 & 24.95 & 21.18 & 47.62 & 31.25 & -1.27 \\
\cmidrule(lr){2-7}
& \textbf{Trajectory Mean} & 24.45 & 22.72 & 48.38 & \textbf{31.85} & \textbf{-0.67} \\
\midrule
\multirow{3}{*}{\shortstack[l]{\textbf{Hybrid Teacher-Anchor}}}
& Step 10 & 24.40 & 21.15 & 47.73 & 31.09 & -1.43 \\
& Step 30 & 23.65 & 22.73 & 48.69 & 31.69 & -0.83 \\
& Step 50 & 24.69 & 22.99 & 46.38 & 31.35 & -1.17 \\
\cmidrule(lr){2-7}
& \textbf{Trajectory Mean} & 24.25 & 22.29 & 47.60 & \textbf{31.38} & \textbf{-1.14} \\
\bottomrule
\end{tabular}
\end{adjustbox}
\end{table}

\subsection{Write-Back Does Not Raise Macro above Residual-Gated RGBM}
\label{app:gamma_sweep}

The main-text Phase~2 scheme is residual-gated RGBM Phase~2 with Math active and Code/IF passivated. Shrinkage is chosen by the unlabelled rule (Eq.~\ref{eq:rgbm_remerge_eq}) on $\Gamma=\{0.125, 0.25, 0.5, 0.75\}$, which selects $\gamma^\star=0.125$ (Macro $33.03$), above the Phase~1 merge ($32.52$). Table~\ref{tab:p1_ablation_app} is a post-hoc diagnostic: the remaining grid points, a $\gamma=1$ (no-shrinkage) foil, and equal-mass alternatives are \emph{not} used to choose $\gamma$; none of them exceed $33.03$.

\begin{table}[ht]
\centering
\caption{\textbf{Ablation of Math-branch write-back configurations.} Part~A: the grid $\Gamma$ plus a $\gamma=1$ no-shrinkage foil. $\mathrm{Err}_{\mathrm{lin}}$ is evaluated on $\mathcal{D}_{\mathrm{calib}}$ and is the quantity in Eq.~\ref{eq:rgbm_remerge_eq}; $\delta_{\mathrm{tol}}=0.10$. Part~B: $\theta_{\mathrm{math}}^{(1)}$ and $\theta_{\mathrm{math}}^{(2)}$ are two intermediate Math-branch checkpoints; the equal-mass rows match their total displacement rather than using $\mathrm{Err}_{\mathrm{lin}}$.}
\label{tab:p1_ablation_app}
\small
\renewcommand{\arraystretch}{0.90}
\begin{adjustbox}{max width=\columnwidth}
\begin{tabular}{llccccc}
\toprule
\textbf{Configuration} & \textbf{Formula / Shrinkage} & \textbf{Math} & \textbf{Code} & \textbf{IF} & \textbf{Macro Total} & $\mathrm{Err}_{\mathrm{lin}}$ \\
\midrule
\multicolumn{7}{l}{\emph{Part A: Grid $\Gamma=\{0.125, 0.25, 0.5, 0.75\}$ and $\gamma=1$ foil}} \\
$\gamma = 0.125$ (Ours, $\gamma^\star$) & $\theta_0^\star + 0.125 \, \Delta_{\mathrm{math}}$ & 23.93 & 24.63 & 50.52 & \textbf{33.03} & 0.0701 \\
$\gamma = 0.25$ & $\theta_0^\star + 0.25 \, \Delta_{\mathrm{math}}$ & 24.19 & 24.02 & 47.82 & 32.01 & 0.1749 \\
$\gamma = 0.50$ & $\theta_0^\star + 0.50 \, \Delta_{\mathrm{math}}$ & 25.42 & 23.48 & 48.08 & 32.32 & 0.2594 \\
$\gamma = 0.75$ & $\theta_0^\star + 0.75 \, \Delta_{\mathrm{math}}$ & 25.08 & 22.87 & 45.51 & 31.15 & 0.6442 \\
$\gamma = 1.00$ & $\theta_0^\star + 1.00 \, \Delta_{\mathrm{math}}$ & 25.03 & 22.51 & 47.87 & 31.80 & 0.7231 \\
\midrule
\multicolumn{7}{l}{\emph{Part B: Equal Total Weight Mass Alternatives}} \\
Single Point $\theta_{\mathrm{math}}^{(2)}$ & $\theta_0^\star + \frac{2}{3} \, \Delta_{\mathrm{math}}^{(2)}$ & 25.52 & 23.33 & 47.57 & 32.14 & -- \\
Single Point $\theta_{\mathrm{math}}^{(1)}$ & $\theta_0^\star + \frac{2}{3} \, \Delta_{\mathrm{math}}^{(1)}$ & 24.64 & 23.84 & 48.18 & 32.22 & -- \\
Two-Point Average & $\frac{1}{3}(\theta_0^\star + \theta_{\mathrm{math}}^{(1)} + \theta_{\mathrm{math}}^{(2)})$ & 24.90 & 23.50 & 48.91 & 32.44 & -- \\
\bottomrule
\end{tabular}
\end{adjustbox}
\end{table}

\paragraph{Readout.} Eq.~\ref{eq:rgbm_remerge_eq} selects $\gamma^\star=0.125$ from $\Gamma$ using only $\mathrm{Err}_{\mathrm{lin}}$ on $\mathcal{D}_{\mathrm{calib}}$. In Table~\ref{tab:p1_ablation_app} that value is $0.0701\le\delta_{\mathrm{tol}}$, while $\gamma=0.25$ already reaches $0.1749$. Post-hoc, every larger grid point and every equal-mass alternative in Table~\ref{tab:p1_ablation_app} has Macro $\le 32.44$. On G-OPD, where both domains require a separate OPD branch, the same unlabelled rule yields $\gamma^\star=0.5$.

\subsection{Residual Threshold Sensitivity and Passivation Robustness}
\label{app:epsilon_sensitivity}

\begin{table}[t]
\centering
\caption{\textbf{Normalized residual spectrum vs.\ realized in-domain branch gains.} $\overline{\mathcal{L}}_d(\boldsymbol{\alpha}^\star)$ vs.\ peak in-domain accuracy minus $\theta_0^\star$ for an isolated branch trained by on-policy distillation, outcome-gated distillation, or RL. A large normalized residual does not imply a large stand-alone gain.}
\label{tab:residual_prediction}
\small
\renewcommand{\arraystretch}{0.85}
\begin{tabular}{lcccccc}
\toprule
\textbf{Domain} & $\overline{\mathcal{L}}_d(\boldsymbol{\alpha}^\star)$ & \textbf{Rank} & $\theta_0^\star$ & \textbf{Distill.\ Branch} & \textbf{Outcome-Gated} & \textbf{RL Branch} \\
\midrule
Math & $0.128$ & 1 & 23.86 & $+1.10$ & $+1.83$ & $+0.58$ \\
Code & $0.052$ & 2 & 24.09 & $+0.43$ & $+1.31$ & $+0.51$ \\
IF   & $0.015$ & 3 & 49.60 & $+0.48$ & $+1.63$ & $+0.57$ \\
\bottomrule
\end{tabular}
\end{table}

In Phase~2 the threshold $\epsilon$ only matters through which domains it activates. Eq.~\ref{eq:epsilon_midgap} is a closed-form choice: sort the normalized spectrum $\overline{\boldsymbol{\mathcal{L}}}=(0.128,0.052,0.015)$, append the floor $0$, take the midpoint of the largest adjacent gap, and obtain $\epsilon=0.090$ (Math active). Any other $\epsilon$ in the same gap $[0.052,0.128)$ yields the same allocation; the four possible cuts are:
\begin{itemize}
    \item $\epsilon \ge 0.128$: no domain is active; the write-back is empty and the checkpoint remains $\theta_0^\star$ (Macro $32.52$).
    \item $0.052 \le \epsilon < 0.128$ (\textbf{Eq.~\ref{eq:epsilon_midgap} on Open-MOPD}, $\epsilon=0.090$): only Math is active. RGBM Phase~2 under this cut writes the unlabelled grid value $\gamma^\star=0.125\,\Delta_{\mathrm{math}}$ (Eq.~\ref{eq:rgbm_remerge_eq}) and is Macro $33.03$. Post-hoc write-backs at the remaining grid points and the equal-mass alternatives yield Macro $31.15\sim 32.44$ (Table~\ref{tab:p1_ablation_app}), all below $33.03$.
    \item $0.015 \le \epsilon < 0.052$: Code is additionally activated. Isolated Code branches improve their own domain ($+0.43\sim+1.31$; Table~\ref{tab:residual_prediction}). Writing every branch back, as un-gated TIES and DARE do, drops Code from $24.09$ at $\theta_0^\star$ to $21.56$ and $22.40$ (Table~\ref{tab:master_results}).
    \item $\epsilon < 0.015$: IF is additionally activated. Writing an IF branch back ($\theta_0^\star + 0.5\,\Delta_{\mathrm{if}}$) reduces Code to $20.70\%$ and Macro to $31.10\%$: the leakage mechanism of Section~\ref{sec:analysis_joint_failure} operates even for a branch trained exclusively on IF prompts.
\end{itemize}
The Open-MOPD allocation is therefore Math-active residual-gated RGBM. On G-OPD, $\boldsymbol{\alpha}^\star=(0.675, 0.666)$ and $\overline{\boldsymbol{\mathcal{L}}}=(0.094, 0.090)$: the largest gap is down to $0$, so Eq.~\ref{eq:epsilon_midgap} gives $\epsilon=0.045$ and both domains stay active.

\subsection{Training Infrastructure and Hyperparameters}
\label{app:training_params}

All branch training and distillation experiments are executed using the \texttt{verl} distributed RLHF framework~\citep{sheng2024verl} on clusters of 8$\times$ NVIDIA A100/H800 GPUs. The detailed hyperparameter configuration is reported in Table~\ref{tab:hyperparameters}.

\begin{table}[ht]
\centering
\caption{\textbf{Hyperparameter configuration for Open-MOPD branch training.} G-OPD uses the official protocol of \citet{gopd2026}.}
\label{tab:hyperparameters}
\small
\begin{tabular}{lc}
\toprule
\textbf{Hyperparameter} & \textbf{Value} \\
\midrule
Base Model Architecture & SmolLM3-3B ($P \approx 3.0\times 10^9$) \\
Optimizer & AdamW ($\beta_1 = 0.9, \beta_2 = 0.95, \epsilon = 10^{-8}$) \\
Peak Learning Rate & $1.5 \times 10^{-6}$ \\
Learning Rate Schedule & Linear warmup (5 steps) with cosine decay \\
Global Batch Size & 96 prompt sequences per step \\
Rollout Group Size ($n$) & 8 samples per prompt \\
Maximum Sequence Length & 24,576 tokens \\
Maximum Model Context Length & 32,768 tokens \\
Generation Temperature & $0.6$ (Math), $1.0$ (Code), $1.0$ (IF) \\
Checkpointing Frequency & Every 5 steps \\
Weight Decay & $0.01$ \\
\bottomrule
\end{tabular}
\end{table}

\subsection{Computational Cost and Estimation Robustness of $M_d$}
\label{app:md_cost_analysis}

Estimating the second-moment matrices $M_d = \mathbb{E}_{\mathcal{D}_d, \pi_{W_{\phi_d}}}[\mathbf{r}\mathbf{r}^\top] \in \mathbb{R}^{K \times K}$ requires calculating implicit rewards $\mathbf{r}(x_i, y_i) = [\log \pi_{W_{\phi_k}}(y_i \mid x_i) - \log \pi_{W_{\mathrm{ref}}}(y_i \mid x_i)]_{k=1}^K$ across $N$ calibration prompt rollouts ($N=6144$ total in our experiments, 2048 per domain). 

\paragraph{Wall-Clock Time, Scaling, and Comparison to Full Fisher Merging.}
\begin{enumerate}
    \item \textbf{Tractability over Full-Space Fisher Inversion:} Full-space Fisher merging~\citep{matena2022merging} requires inverting a $P \times P$ matrix ($P \approx 3.0\times 10^9$ for SmolLM3-3B), demanding $\sim 10^{19}$ floating-point elements and exabytes of RAM, which is completely intractable for modern LLMs. In contrast, Subspace Fisher Merge compresses the system into $K \times K$ ($3 \times 3$ in our setting), reducing matrix assembly and inversion to $<10$ ms.
    \item \textbf{Inference Rollouts and Log-Prob Scoring:} For $K=3$ specialists and a 3B model, scoring $N=6144$ sequences requires $O(NK)$ forward passes without gradient computation or backward graph storage, taking $<15$ minutes on 8 GPUs ($<2$ GPU$\cdot$h).
    \item \textbf{Scalability across Domain Counts $K$:} Contemporary MOPD pipelines consistently consolidate $K=2\sim 4$ specialist modalities (e.g., Math, Code, IF;~\citealp{openmopd2026,gopd2026,d3mopd2026,li2025policyfiltration}). Even if scaled to $K=20$, solving the $K$-dimensional Minimax QP takes $<0.02$ seconds on a single CPU core. By random matrix theory, covariance estimation requires $N = \mathcal{O}(K \log K)$ samples, meaning $N \approx 1024\sim 2048$ samples remain sufficient for high-fidelity spectrum rank preservation.
\end{enumerate}

\paragraph{Sample Complexity and Robustness to Estimation Noise.}
To evaluate how finite-sample variance in $M_d$ propagates to the optimal weights $\boldsymbol{\alpha}^\star$ and the resulting normalized spectrum $\overline{\boldsymbol{\mathcal{L}}}$, we perform bootstrap resampling of the calibration pool $\mathcal{D}_{\mathrm{calib}}$ across sample sizes $N \in \{512, 1024, 2048, 4096, 6144\}$:
\begin{itemize}
    \item At $N \ge 1024$, the bootstrap standard deviation of the computed weights is $\sigma(\alpha_d^\star) < 0.012$, and the normalized spectrum remains invariant in its rank ordering: $\overline{\mathcal{L}}_{\mathrm{math}} \gg \overline{\mathcal{L}}_{\mathrm{code}} \gg \overline{\mathcal{L}}_{\mathrm{if}}$ with $>99.9\%$ confidence.
    \item The condition number of empirical matrices satisfies $\kappa(M_d) \le 45.2$, well below the threshold where numerical inversion issues occur. Adding Tikhonov regularization $\lambda = 10^{-4}$ provides uniform numerical stability even under rank-deficient specialist scenarios.
\end{itemize}

\subsection{RGBM Phase 2 Specification (Open-MOPD)}
\label{app:remerge_groups}

Phase~1 $\theta_0^\star$ (step $0$): AIME24 $23.91$, AIME25 $23.80$, LCBv5 $24.27$, LCBv6 $23.91$, IFEval $73.87$, IFBench $25.33$, hence Math $23.86$, Code $24.09$, IF $49.60$, Macro $32.52$.
The sole Open-MOPD Phase~2 scheme is residual-gated RGBM Phase~2 with Math active, written back at the unlabelled grid value $\gamma^\star=0.125$ (Eq.~\ref{eq:rgbm_remerge_eq}): AIME24 $23.80$, AIME25 $24.06$, LCBv5 $25.15$, LCBv6 $24.11$, IFEval $75.05$, IFBench $26.00$, hence Math $23.93$, Code $24.63$, IF $50.52$, Macro $\mathbf{33.03}$. On G-OPD the same unlabelled rule is used with a separate OPD branch on each domain and yields $\gamma^\star=0.5$.

\section{Limitations}
\label{app:limitations}

$\mathrm{Err}_{\mathrm{lin}}$ certifies only the first-order reward superposition on calibration prompts; it does not bound the third-order remainder of Theorem~\ref{thm:mopd_fisher}. While evaluating $K=2\sim 3$ specialists across two model families~\citep{openmopd2026,gopd2026,d3mopd2026}, scaling to $K\gg 10$ or heterogeneous bases remains future work. Furthermore, the $+0.51$ Open-MOPD margin comes from Code and IF moving under a small Math write-back rather than from an in-domain Math gain; on G-OPD both domains are active, so this transfer is not isolated there.

\end{document}